\documentclass{article}

\usepackage{microtype}
\usepackage{graphicx}
\usepackage{booktabs} % for professional tables
\usepackage{upgreek}
\usepackage[utf8]{inputenc}
\usepackage[american]{babel}
\usepackage{pgfplots}
\pgfplotsset{compat=newest}
\usepackage{tikz}
\usepackage{mathtools}
\usetikzlibrary{shapes.geometric}

\usepackage{amssymb}
\usepackage{placeins}
\usepackage{graphicx}
\usepackage{amsmath}
\usepackage{amsfonts}
\usepackage{epsfig}
\usepackage{color}
\usepackage{stmaryrd}
\usepackage{multirow}
\usepackage{subcaption}
\usepackage{lipsum}
\usepackage{dsfont}
\usepackage{float}

\newcommand{\un}{\mathds{1}}

\def\E{\mathbb E}
\def\EXP{{\E}}
\def\expec{{\EXP}}

\def\ol{\overline}

\newcommand{\defeq}{\stackrel{\rm def}{=}}

\def\argmax{\mathop{\rm arg\, max}}
\def\argmin{\mathop{\rm arg\, min}}

\def\blackslug{\hbox{\hskip 1pt \vrule width 4pt height 8pt depth 1.5pt
\hskip 1pt}}
\def\qed{\quad\blackslug\lower 8.5pt\null\par}

\def\C{{\cal C}}

\newcommand{\RR}{\mathbb{R}}
\newcommand{\rset}{\RR}

\newcommand{\Prob}[1]{\mathbb{P}\left( #1 \right) }
\newcommand{\PP}{\Prob}
\newcommand{\given}[1][{}]{\;\middle\vert\;{#1} }

\usepackage{microtype}
\usepackage{amssymb}
\usepackage{hyperref}
\newtheorem{theorem}{{\bf Theorem}}

\newtheorem{proposition}{{\bf Proposition}}
\newtheorem{definition}{{\bf Definition}}

\newtheorem{lemma}{{\bf Lemma}}

\newtheorem{remark}{{\bf Remark}}

\usepackage{inputenc}
\usepackage{standalone}

\newlength\Radius
\usepackage{hyperref}

\usepackage[accepted]{icml2021}

\icmltitlerunning{Feature Clustering for Support Identification in Extreme Regions}

\begin{document}

\twocolumn[
\icmltitle{Feature Clustering for Support Identification in Extreme Regions}




\icmlsetsymbol{equal}{*}

\begin{icmlauthorlist}
\icmlauthor{Hamid Jalalzai}{tp}
\icmlauthor{Rémi Leluc}{tp}
\end{icmlauthorlist}

\icmlaffiliation{tp}{Télécom Paris, Institut Polytechnique de Paris, France}
\icmlcorrespondingauthor{Hamid Jalalzai}{hamid.jalalzai@gmail.com}
\icmlcorrespondingauthor{Rémi Leluc}{remi.leluc@gmail.com}

\icmlkeywords{}
\vskip 0.3in
]



\printAffiliationsAndNotice{}  % leave blank if no need to mention equal contribution

%Rajouter qu'on veut écrire la mesure comme une somme sur les mesures restreintes aux clusters.}
%Capturing the dependence structure of multivariate extreme data is a major challenge in many fields involving the management of risks that come from multiple sources, \textit{e.g.}, portfolio monitoring, environmental risk management, insurance and anomaly detection.

\begin{abstract}
Understanding the complex structure of multivariate extremes is a major challenge in various fields from portfolio monitoring and environmental risk management to insurance. In the framework of multivariate Extreme Value Theory, a common characterization of extremes' dependence structure is the angular measure. It is a suitable measure to work in extreme regions as it provides meaningful insights concerning the subregions where extremes tend to concentrate their mass. The present paper develops a novel optimization-based approach to assess the dependence structure of extremes. This support identification scheme rewrites as estimating \emph{clusters of features} which best capture the support  of extremes. The dimension reduction technique we provide is applied to statistical learning tasks such as feature clustering and anomaly detection. Numerical experiments provide strong empirical evidence of the relevance of our approach.
\end{abstract}

%%%%%%%%%%%%%%%%%%%%%%%%%%%%%%
\section{Introduction}
%{\color{blue}{\url{https://www.researchgate.net/publication/23551907_On_portfolio_selection_under_extreme_risk_measure_The_heavy-tailed_ICA_model}}}

%Clustering is essential for exploratory data mining, data structure analysis and a common technique for statistical data analysis. It is widely used in many fields, including machine learning, pattern recognition, image analysis, information retrieval, bioinformatics, data compression, and computer graphics. %It gave birth to multiple approaches which are different so such the intrinsect notion of cluster varies.
%Many clustering approaches exist with different intrinsic notions of what a cluster is. In the standard setup, the goal is to group objects into subsets, known as clusters, such that objects within a given cluster are more related to one another than the ones from a different cluster.

In a wide variety of applications ranging from structural engineering to finance, \textit{extreme} events can occur with a far from negligible probability \cite{embrechts1999extreme, embrechts2013modelling}. In the multivariate setting, such events are usually modeled through threshold exceedance. A random vector $X = (X^1, \ldots, X^p) \in \rset^p, (p > 1)$ is said to be extreme if $\|X\| > t$ for any given norm $\|\cdot\|$ and some \emph{large} threshold $t > 0$. The latter is generally chosen so that a small but non negligible proportion of data falls in the extreme regions $\{x \in \rset^p, \|x\| > t\}$.
In machine learning tasks, it is relevant to apply different treatments to \textit{extreme} and \textit{normal} data. Devoting attention to extreme regions can lead to better understanding of the distributional law of $X$ and practical performance of classical algorithms, as shown by several recent studies: in anomaly detection \citep{Roberts99,Clifton2011, goix2016sparse,thomas2017anomaly}, classification \citep{vignotto2018extreme,jalalzai2018binary,jalalzai2020heavy} or feature clustering \citep{chautru2015dimension, chiapino2019identifying,janssen2020k} when dedicated to the most extreme regions of the sample space.

Scaling up multivariate Extreme Value Theory (EVT) is a key issue %that one faces
when addressing high-dimensional learning tasks. Indeed, most multivariate extreme value models have been designed to handle moderate dimensional problems, \textit{e.g.}, where dimension $p \leq 10$. For larger dimensions, simplifying modeling choices are required, stipulating for instance that only some predefined subgroups of components may be concomitant extremes, or, on the contrary, that all must be \citep{stephenson2009high,sabourinNaveau2012}.% This curse of dimensionality can be explained, in the context of extreme values analysis, by the relative scarcity of extreme data, the computational complexity of the estimation procedure and, in the parametric case, by the fact that the dimension of the parameter space usually grows with that of the sample space.
This calls for dimensionality reduction devices adapted to multivariate extreme values.

Identifying the features $X^j$'s (and the resulting subspaces) contributing to $X$ being extreme is a major challenge in EVT. The distributional structure of extremes highlights the components of a multivariate random variable that may be simultaneously large while the others remain small. This is a valuable piece of information for multi-factor risk assessment or detection of anomalies among other –not abnormal– extreme data. Two phenomena are likely to happen: \textit{(i)} only a small number of features may be concomitantly large, so that only a small number of subspaces have non-zero mass, \textit{(ii)} each of these groups -\textit{clusters of features}- contains a limited number of coordinates (compared to the original dimensionality), so that the corresponding subspace with non zero mass have small dimension compared to $p$.
The purpose of this paper is to introduce a data-driven methodology for identifying such subspaces, to reduce the dimensionality of the problem and thus to learn a sparse representation of extreme behaviors.

\begin{figure}
    \centering
    \includegraphics[trim={0 2.5cm 0 0.2cm},clip, width=0.5\textwidth]{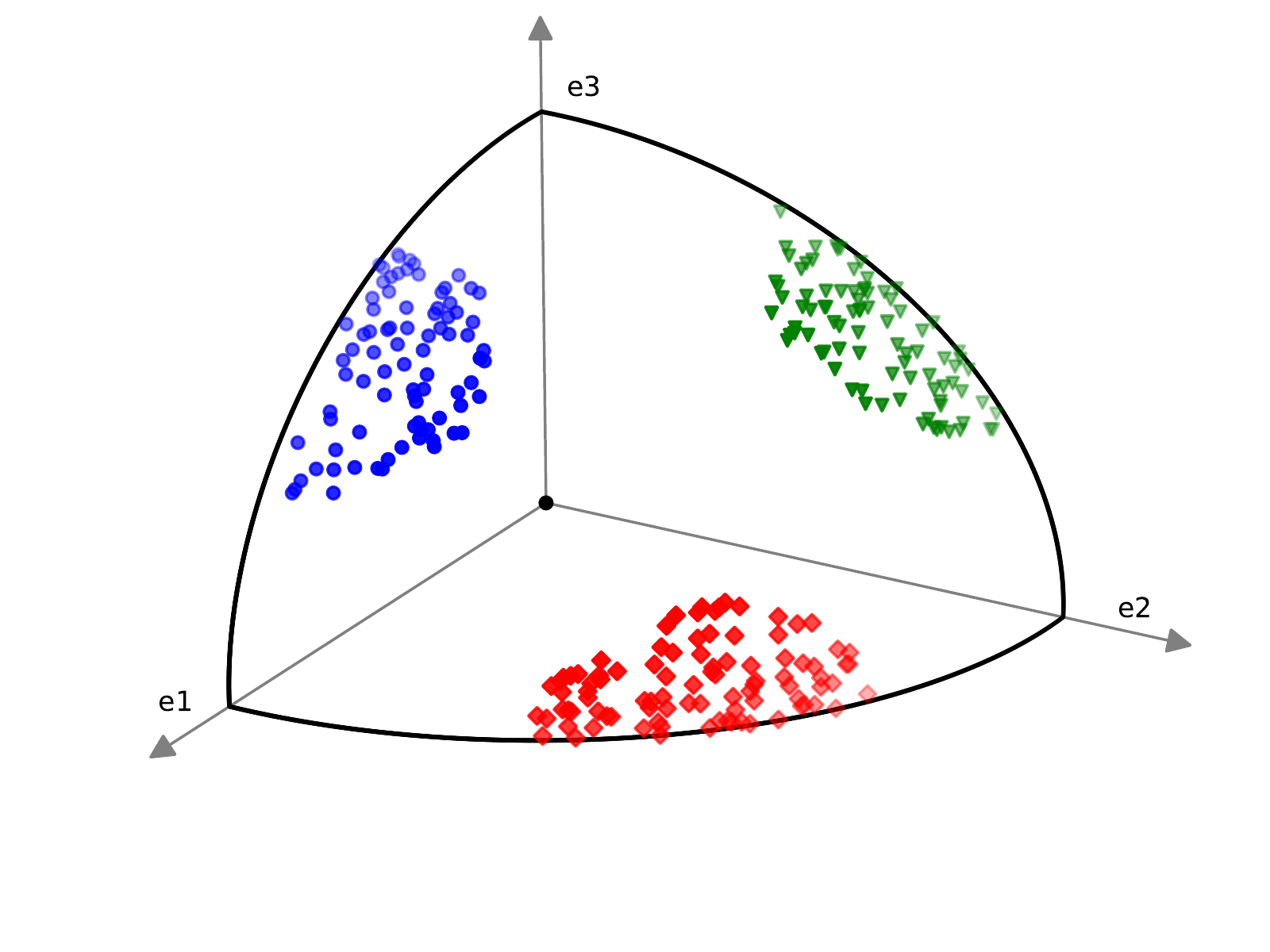}
    %trim cause the image is too long, we remove 2cm at the bottom
    \caption{Illustration of  normalized extremes $\theta_i$'s on the $\ell_2$-sphere of $\rset_+^3$ with clusters of features $K_1=\{1,3\}$(blue), $K_2=\{1,2\}$(red) and $K_3=\{2,3\}$(green).} %Two clusters are represented: blue dots lying on the $\ell_2$ sphere represent the angular component of extreme samples where the first and third component are concomitantly large. Orange triangles represent extremes where the first and second get large simultaneously. Lastly an outlier extreme sample, lying on the sphere, is represented by a red star. }
    \label{fig:sphere}
\end{figure}

 %Recalling the framework of \citet{chautru2015dimension, chiapino2016feature, chiapino2019identifying},
 This paper provides a novel optimization-based approach for finding subspaces from multivariate extreme features. Given $n \geq 1$ \textit{i.i.d} copies  $X_1,\ldots,X_n$ of a heavy-tailed random variable $X=(X^1,\ldots,X^p) \in \rset^p$, the goal is to identify clusters of features $K \subset \llbracket 1,p \rrbracket$ such that the variables $\{X^j: j \in K\}$ may be large while the other variables $X^j$ for $j \notin K$ simultaneously remain small. Figure \ref{fig:sphere} depicts such an example of normalized extremes with the associated feature clusters.
\newpage
 Up to approximately $2^p$ combinations of extreme features are possible and contributions such as \citet{chautru2015dimension, chiapino2016feature,  goix2016sparse,engelke2018graphical,chiapino2019identifying} tend to identify a smaller number of simultaneous extreme features. Dimensional reduction methods such as principal components analysis and derivatives \citep{wold1987principal, cutler1994archetypal,tipping1999probabilistic,cooley2019decompositions, drees2019principal} can be designed to find a lower dimensional subspace where extremes tend to concentrate. Following this path, the idea of the present paper is to decompose the $\ell_1$-norm of a positive input sample as a weighted sum of its features. % Another way of identifying the clusters of features that may jointly be large is to select \emph{combinations of extreme features}, in the spirit of archetypes defined by \citet{cutler1994archetypal} (which do not apply in our framework).
%\textbf{Related work.} Dimension reduction, Damex, CLEF
%Dimension reduction PCA in the extreme \citep{drees2019principal}, Cutler \citep{cutler1994archetypal},
%Clustering: \citep{chiapino2016feature, chiapino2019identifying, chautru2015dimension}
%Since the approach described in this paper relates to finding a manifold of the extremes within the input space
%\lipsum[1-1]
%\textbf{Contributions.}

Several EVT contributions are aimed at assessing a sparse support of multivariate extremes \citep{de2007extreme,chiapino2016feature,meyer2019sparse,engelke2020sparse}. A broader scope of contributions related to the work detailed in this paper ranges from compressed sensing \citep{candes2006robust, candes2006stable,tsaig2006extensions} and matrix factorization \citep{lee2001algorithms,7254164} to group sparsity \citep{yuan2006model,simon2013sparse,devijver2015finite}.

\textbf{Contributions.}
The main results of this paper are: \\
\textit{(i)} We present a novel optimization-based approach to perform subspace clustering of extreme regions in the multivariate framework. This is achieved by the algorithm \textit{Multivariate EXtreme Informative Clustering by Optimization} (in  short MEXICO) which finds a sparse representation for the dependence structure of extremes. \\
\textit{(ii)} Following contribution laid out by \citet{niculae2018sparsemap}, we study at length different manifolds on the probability simplex, including our $\mathbb{M}$-\textit{set}. Our analysis may be of independent relevance. \\
\textit{(iii)} The performance of the introduced algorithm are demonstrated from both theoretical and empirical points of view. First we provide a non-asymptotic bound on the excess risk. Secondly, numerical experiments on both \emph{feature clustering} and \emph{anomaly detection} tasks in extreme regions  demonstrate the relevance of our method when compared to existing methods.

%%%%s. The aim of the present work is to extend[30]’s methodology to datasets which do not satisfy their assumptions, in particular to text datasetsembedded by state of the art techniques. This is achieved by the algorithmLearning a Heavy TailedRepresentation(in shortLHTR) which learns a transformation mapping the input dataXonto arandom vectorZwhich does satisfy the aforementioned assumptions. The transformation is learnt byan adversarial strategy [26]

%The main results of this paper are: \textit{(i)},  we present a novel optimization-based approach to perform clustering of extreme features in the multivariate set-up with respected regularity property

\textbf{Notations.} The following notations are used throughout the paper: $\mathcal{M}_{n}^p([1,+\infty[)$ is the set of $n \times p$ matrices valued in $[1,+\infty[$. Any matrix is denoted in bold. $\mathcal{A}_p^m$ denotes the set of \textit{mixture matrices} composed of $p \times m$ matrices valued in $[0,1]$ where the sum of elements of any column equals $1$. For any $\mathbf{M}=(\mathbf{M}_i^j) \in \mathcal{M}_{n}^p(\rset)$, for $i \in \llbracket1,n\rrbracket$ (\textit{resp.} $j \in \llbracket1,p\rrbracket$), let $e_i$ (\textit{resp.} $e^j$) denote the vector of the canonical basis such that $e_i \mathbf{M}  = \mathbf{M}_i$ (\textit{resp.} $\mathbf{M} e^j = \mathbf{M}^j$) where $M_i$ corresponds to the $i$-th line of $\mathbf{M}$ (\textit{resp.} $M^j$ corresponds to the $j$-th column). Denote $\mathfrak{S}_m$ the finite symmetric group of order $m$. Let $E = [0, \infty]^p \backslash\{0\}$ and $\Omega_{p, ||\cdot||} = \{x \in \rset^p_+ : \|x\| \leq 1\}$ the ball associated to the norm $||\cdot||$ and its complementary set $\Omega_{p, ||\cdot||}^c = \rset^p_+ \backslash \Omega_p$, let $S$ denote the sphere associated to $\|\cdot\|$ and for $x \in \rset^p$ and $K \subset \llbracket 1,p \rrbracket$, write $x^{(K)} = (x^j \un_{j \in K})$. Denote by $\Gamma$ the Euler function.

\textbf{Outline.} The paper is organized as follows, in Section \ref{sec:EVT} we introduce the multivariate EVT background and our problem of interest. In Section \ref{sec:optim} we present our optimization-based approach along with its specific details concerning the projection step onto the probability simplex. Section \ref{sec:th_results} gathers the theoretical results. We perform some numerical experiments in Section \ref{sec:exp} to highlight the performance of our method and we finally conclude in Section \ref{sec:conclusion}. Proofs, technical details and additional results can be found in the appendix.

%%%%%%%%%%%%%%%%%%%%%%%%%%%%%%
\section{Preliminaries } \label{sec:EVT} %\& Probabilistic Framework

Extreme value theory develops models for learning the unusual
rather than the usual, in order to provide a reasonable assessment of the
probability of occurrence of \emph{rare events}. This section first recalls the required mathematical framework and classical tools for the analysis of multivariate extremes and then introduce our problem of interest.
\subsection{Mathematical background}

%Such models are widely used in fields involving risk management such as finance, insurance, operation research, telecommunication or environmental sciences.
%\textbf{Mathematical background.}
The notion of \textit{regular variation} is a natural way for modelling power law behaviors that appear in various fields of probability theory. In this paper, we shall focus on the dependence and regular variation of random variables and random vectors. We refer to the book of \citet{Resnick1987} for an excellent account of  heavy-tailed distributions and the theory of regularly varying functions.

\begin{definition}{(Regular variation \cite{karamata1933mode})}
A positive measurable function $g$ is regularly varying with index $\alpha \in \rset$, notation $g \in \mathcal{R}_{\alpha}$ if $\lim_{x \to +\infty} g(tx)/g(x) = t^{\alpha}$ for all $t>0$.
\end{definition}
The notion of regular variation is defined for a random variable $X$ when the function of interest is the distribution tail of $X$.

\begin{definition}(Univariate regular variation)
A non-negative random variable X is regularly varying with tail index $\alpha \geq 0$ if its right distribution tail $x \mapsto \PP{X>x}$ is regularly varying with index $-\alpha$, i.e., $\lim_{x \to +\infty} \PP{X>tx \given X>x } = t^{-\alpha}$ for all $t>1$.
\end{definition}

This power-law behavior may be thought of as a smoothness condition for the tail at infinity. This definition can be extended to the multivariate setting where the topology of the probability space is involved. We rely on the vague convergence of measures \citep[Section 3.4]{Resnick1987} and consider the following definition \citep[p.69]{resnick1986point}.

\begin{definition}(Multivariate regular variation)
\label{def:multivariateRegularVariation}
A random vector $X \in \rset_{+}^p$ is regularly varying with tail index $\alpha \geq 0$ if there exists $g \in \mathcal{R}_{-\alpha}$ and a nonzero Radon measure $\mu$ on $E$ such that
\begin{align*}
g(t)^{-1} \PP{t^{-1} X \in A}  \xrightarrow[t \to \infty]{} \mu(A),
\end{align*}
where $A \subset E$ is any Borel set such that $0 \not \in \partial A$ and $\mu(\partial A) = 0$.
\end{definition}
The limiting measure $\mu$, known as the \emph{exponent measure}, is homogeneous of order $-\alpha$ \textit{i.e.} for any $t > 0$, $\mu(t\cdot) = t^{-\alpha}\mu(\cdot) $. This suggests a polar decomposition of $\mu$ into a radial component and an angular component $\Phi$. For any $x = (x_1, \cdots, x_p) \in \rset^p_+$, one can set
\begin{align*}
\left\{
    \begin{array}{ll}
        R(x)&=||x|| \\
        \Theta(x)&=\left(\frac{x_1}{R(x)}, \ldots, \frac{x_p}{R(x)}\right) \in S
    \end{array}
\right.
\end{align*}
For any $B \subset S$, the \textit{angular measure} $\Phi$ on $S$ is defined as,
%where $S$ is the positive orthant of the unit sphere in $\rset^d$ for some norm $||\cdot||$.
\begin{align*}
    \Phi(B) \defeq \mu(\{x, R(x) \geq 1, \Theta(x) \in B \}).
\end{align*}
The angular measure $\Phi$ plays a central role in the analysis of extremes, as it characterizes the directions where extremes are more likely to occur. Assessing the support of $\Phi$, or equivalently of $\mu$, leads to forecasting the directions where extremes are more likely to occur \textit{i.e.} features that are more likely to jointly be large.

\subsection{Probabilistic Framework \& Problem Statement}

%\textbf{Selection of extremes.}
We observe $n \geq 1$ \textit{i.i.d} copies  $X_1,\ldots,X_n$ of a regularly varying random vector $X=(X^1,\ldots,X^p) \in \rset^p$ with tail index $\alpha=1$. Extremes correspond to samples with norm larger than a fixed threshold $t > 0$. Incidentally, $t$ should depend on $n$, as the notion of \textit{extreme} should be understood as \textit{large} norms compared to the vast majority of observed data. %In practice, the observations are sorted by decreasing order of magnitude $\|X_{(1)}\| \geq \ldots \geq \|X_{(n)}\|$ and we consider the $k$ largest vectors
%$X_{(1)},\ldots,X_{(k)}$
%yielding to $k$ samples considered as extremes.
The Euclidian space $\rset^p$ being of finite dimension, all norms are equivalent and the choice of the norm does not matter for the definition of the limit measure \citep{beirlant2006statistics}, therefore we may use the $\ell_\infty$-norm in the remainder of this paper to analyse extremes. In other words $R(x)$  is set as $||.||_\infty$ in Definition \ref{def:multivariateRegularVariation}.
The observations are first sorted by decreasing order of magnitude $\|X_{(1)}\|_\infty \geq \ldots \geq \|X_{(n)}\|_\infty$. Then, consider a small fraction $0 < \gamma < 1$ of the observations and denote by $t_\gamma$ the quantile of $||X||_\infty$ at level $(1-\gamma)$, \textit{i.e.} $\PP{||X||_\infty > t_\gamma} = \gamma$. The extreme samples are $X_{(1)},\ldots,X_{(k)}$ where $k = \lfloor n\gamma\rfloor$ is a selection threshold (cf Remark \ref{rem:sel_k}).

\begin{remark}(Pareto Standardization)
  \label{rmk:ParetoStandardization}
  In this work, it is assumed that all marginal distributions are tail equivalent to the Pareto distribution with index $\alpha = 1$. In practice, the tails of the marginals may be different and it is convenient to work with marginally standardized variables. Thus, the margins $F^j (x^j) = \PP{X^j \leq x^j }$ are separated from the dependence structure in the description of the joint distribution of $X$. Consider the Pareto standardized variables $V^j = 1/(1 - F^j (X^j)) \in [1, \infty]$ and
$V = T(X) = (V^1 , \ldots , V^d )$. Replacing X by V permits to take $\alpha = 1$ and $g(t) = 1/t$ in Definition \ref{def:multivariateRegularVariation}. Appendix \ref{appendix:Pareto_scaling} provides further details concerning $\widehat T$ the empirical counterpart of $T$.
\end{remark}

\begin{remark}(On selection of $k$) \label{rem:sel_k}
  Determining $k$ is a central bias variance trade-off of Extreme Value analysis (See \textit{e.g.} \citet{goix2016sparse} and references therein). As $k$ gets too large, a bias is induced by taking into account observations which do not necessarily behave as extremes: their distribution deviates significantly from the limit distribution of extremes. On the other hand, too small values lead to an increase of the algorithm's variance. %In practice, a conventional choice is $k = \sqrt{n}$.
\end{remark}
%Note that this rank standardization is commonly used in multivariate EVT to study the dependence structure of extremes (see \citet{BGTS04} and references therein) and  avoids any further marginal distributions assumptions. The resulting feature variables of  \eqref{eq:rank_scalling} are not independent and the remaining goal is to discover the dependence structure of standardized extremes.

%The asymmetric logistic model is very flexible and can represent a wide range of dependence structures. The parameter $\alpha_{b} \in(0,1]$ measures the strength of dependence for $b \in B .$ The dependence increases as $\alpha_{b}$ decreases. In large dimensions, the full form of the asymmetric logistic model is over-parametrized, but practical applications typically suggest the use of a simpler model that lies within this subclass. For example, in Section 5 we present

%\textbf{Problem statement.}
Our work focuses on assessing the dependence structure in extreme regions in a multivariate setup. The angular measure $\Phi$ fully describes the latter \emph{asymptotic} dependence. Therefore, we seek to accurately infer a \emph{sparse} summary of the mass of extremes spread on each constructed subspace. \\
Let $X \in \rset_{+}^p$ be a multivariate random vector whose dependence structure is unknown. %\textit{e.g.} an asymmetric logistic (see more details in the appendix).
We address the problem of finding different feature clusters $K_j \subset \llbracket 1,p \rrbracket$ with $j=1,\ldots,m$ and $m < p$ such that all features %(at least two)
in a same subset may be large together. %Unit sets are not relevant for the analysis so we assume that each cluster is of size at least 2.
In order to reach a representation of interest, \textit{e.g.}, diversity for portfolio in finance or clusters for smart grids in wireless technologies, we seek disjoint clusters ($K_i \cap K_j = \emptyset$ for all $i \neq j$).
Relying on the $m$ clusters of features $K_1, \ldots, K_m$ and the underlying subspaces, the exponent measure $\mu$ can be approximated as $\mu(\cdot) \approx \sum_{j=1}^m \mu_{K_j} (\cdot)$. Each component $\mu_{K_j}$ is concentrated on the subregion given by the features of cluster $K_j$. %Equivalently, concerning the angular measure $\Phi$, the problem rewrites as $ \Phi(\cdot) \approx \sum_{j=1}^m \Phi_{K_j}(\cdot).$

%Finding the relevant subsets of features $K_1, \ldots, K_m$ such that $\mu$ (\textit{resp.}$\Phi$) is best approximated by $\sum_{j=1}^m\mu_{K_j}$ (\textit{resp.} $\sum_{j=1}^m\Phi_{K_j})$ rewrites in our settings as

 In the remaining of this paper, $\mathbf{X} \in \mathcal{M}_{k, p}([1,+\infty[)$ corresponds to the truncated training set: $X_{(1)},\ldots, X_{(k)}$. We search a subset $K$ of features such that the $\ell_1$-norms of $\mathbf{X}_i$ and its restriction $\widetilde{\mathbf{X}}_{i} = \mathbf{X}_{i}^{(K)}$ are almost equal \textit{i.e.}
 $$\|\widetilde{\mathbf{X}}_i \|_1 \approx \|\mathbf{X}_i\|_1.$$% In a equivalent manner the  analysis can be performed on the angular vectors $\theta_{(1)},\ldots, \theta_{(k)}$ with  $\theta_{(i)} = \Theta(\mathbf{X}_i)$. Let $\mathbf{\Theta} = \{\Theta(\mathbf{X}_i)\}_{i\leq k}$ denote the matrix composed by the normalized samples.
 %||_1 \approx ||\Theta(\widetilde V_i)||_1$, where $\Theta(V) = V/||V||$ for any norm $||\cdot||$ different from the $\ell_1$ norm.

%%%%%%%%%%%%%%%%%%%%%%%%%%%%%%
\section{Feature Mixture in Extreme Regions} \label{sec:optim}

This section presents an approach to find relevant directions of the extreme samples to estimate the support of $\mu$. The analysis is carried out under the empirical risk minimization paradigm and details our algorithm MEXICO.

\subsection{Empirical Risk Minimization}

%\textbf{Empirical Risk.}
To assess the dependence structure of features of extreme samples, we consider the framework of empirical risk minimization focused on extreme regions. Consider an extreme sample $X$, \textit{i.e.}, an observation satisfying $\|X\|_\infty>t_\gamma$. %Using its angular decomposition $\theta = X/\|X\|_{\infty}$,
The goal is to learn a representation function $g: \rset^p \to \rset_{+}$ in order to minimize the Bayes risk at level $t_\gamma$ defined by
% Our goal is to find the dependence structure of features that are jointly large, \textit{i.e.} given an extreme sample $V$, find the possible features that happen to jointly large. A natural risk $\mathcal{R}$ arises is
\begin{equation} \label{eq:def_true_risk}
  \mathcal{R}_{t_\gamma}(g) = \expec_{X}\left[\ell\left(X,g\left(X\right)\right) \Big| \|X\|_\infty>t_\gamma\right],
\end{equation}
where $\ell:\rset_{+}^p \times \rset_{+} \to \rset_{+}$ is a loss function measuring the discrepancy between the true extreme dependence structure of $X$ and its predicted counterpart $g\left(X\right)$. Based on the extreme observations $X_{(1)},\ldots,X_{(k)}$, %and their respective angular components $\theta_{(1)},\ldots,\theta_{(k)}$,
 the empirical risk in the extreme regions is given by
\begin{equation} \label{eq:def_emp_risk}
  \widehat{\mathcal{R}}_k(g) = \frac{1}{k} \sum_{i=1}^k \ell\big(X_{(i)},g(X_{(i)})\big).
\end{equation}

\textbf{Features mixtures.} In order to recover the clusters of features, we consider mixtures of the components of each sample. The true number of subregions is unknown and we search for $m$ clusters where $m$ is selected according to Remark \ref{rem:sel_m}. We consider the probability simplex $\Delta_p$ defined on the positive orthant of $\rset_+^p$ by
\begin{equation*}
\Delta_p = \{x \in \rset^p_+, x_1 + \ldots + x_p = 1\},
\end{equation*}
and let $\mathbf{W} \in \mathcal{A}_p^m$ with $m < p$ be a \textit{mixture matrix}. We denote by $\widetilde{\mathbf{X}} = \mathbf{X}\mathbf{W} \in \mathcal{M}_{k}^m(\rset_+)$ the transformed matrix. The following proposition ensures the preservation of the regular variation of the resulting vectors $\widetilde{X}_i$'s.% and points out the behavior of the limiting measures.

\begin{proposition}[Mixture transformation] \label{th0} Let $X \in \rset_{+}^p$ be a regularly varying vector as defined earlier and $\mathbf{W} \in \mathcal{A}_p^m$ a mixture matrix with $1 < m \leq p$. Then %$V$ is regularly varying and
  the transformed vector $\widetilde{X} = X \mathbf{W} \in \rset_{+}^m$ is regularly varying with tail index $\alpha = 1$. Thereby, if we denote by $\mu$ (resp. $\widetilde \mu$) the limiting measure of $X$ (resp. $\widetilde{X}$), we have $$(1/m)\widetilde \mu (\Omega_{m, ||\cdot||_1}^c) \leq \widetilde \mu (\Omega_{m, ||\cdot||_\infty}^c) \leq \mu (\Omega_{p, ||.||_1}^c).$$
\end{proposition}
The proof of the proposition is deferred to the appendix.
%This follows from Lemma 3.3 in \citep{jessen2006regularly} and
\begin{remark}(Selection of $m$) In view of Proposition \ref{th0}, in practice, the required dimension $m < p$ can be seen as the smallest value $m$ such that the empirical version of $\widetilde \mu (\Omega_{m, ||.||_\infty}^c)$ is arbitrarly close to the empirical version of $\mu (\Omega_{p, ||.||_1}^c)$. Hence, the $m$ selected clusters provide relevant support of extremes.
\label{rem:sel_m}
\end{remark}

%%%%%%%%%%%%%%%%%%%%%%%%%%%%%%%%%%%%%%%%%%%%%%%
\textbf{Loss function.} A natural question rises in the choice of the approximation function $g$ used in Eq. \eqref{eq:def_true_risk}. Each column $\mathbf{W}^j$ for $j \in \llbracket 1,m \rrbracket$ is modelling a mixture of components and represents a cluster $K_j$. For any sample $X$, we want to find a mixture that gives a good approximation in $\ell_1$-norm, \textit{i.e.}, we seek a column $j \in \llbracket 1,m \rrbracket$ for which $\widetilde{{X}}^{j} = (X\mathbf{W})^j$ is the closest to $\|X\|_1$. A simple choice for the approximation function $g$ is reached through a linear combination and defined as follows for any input $x \in \rset^p_+$,
\begin{align}
  \label{eq:gw}
g_{\mathbf{W}}:x \mapsto g_{\mathbf{W}}(x) = \max_{1 \leq j \leq m}(x\mathbf{W})^j.
\end{align}
%In other words, we need to find $j \in \llbracket 1,m \rrbracket$ in order to minimize the score function $\gamma$ defined for all $(i,j) \in \llbracket 1,n \rrbracket \times \llbracket 1,m \rrbracket$ by
%\begin{align*}
%\gamma(\mathbf{W},{V},j) = ||{V}||_1 - \widetilde{{V}}^j.
%\end{align*}
%This approach is similar to \citep{cutler1994archetypal} since the analysis of the dependency structure is assessed through a linear approach.
%Similarly, regarding the angular part $\Theta(x)$, we consider the
The associated loss function is defined by
\begin{align}
  \label{eq:loss_angular}
\ell\left(x,g_\mathbf{W}(x)\right)= \frac{1}{p}\left( \|x\|_1 - g_{\mathbf{W}}(x)\right).
\end{align}
Observe that this particular choice yields a loss function bounded in $[0,1]$ when using the angular decomposition of extremes $\theta = X/\|X\|_{\infty}$. With this choice, the approximation function $g_{\mathbf{W}}$ is parametrized by the mixture matrix $\mathbf{W}$. For ease of notation we abusively denote by $\ell(X,\mathbf{W}) = \ell(X,g_\mathbf{W}(X)) $. Thus, using this specific loss in Eq. \eqref{eq:def_true_risk}, the goal is to learn the mixture matrix $\mathbf{W}_t^\star$  minimizing the risk on extremes
\begin{align*}
\mathbf{W}_t^\star \in \argmin_{\mathbf{W} \in \mathcal{A}_p^m} \left\{ \mathcal{R}_t(\mathbf{W}) \defeq \expec\left[\ell\left(X,\mathbf{W}\right) | \|X\|_\infty > t\right] \right\}.
\end{align*}
Based on the observations $\{X_{(1)},\ldots,X_{(k)}\}$, the optimization problem consists in finding a mixture matrix minimizing the empirical risk
\begin{equation}
\label{original_problem}
\widehat{\mathbf{W}}_k \in \argmin_{\mathbf{W} \in \mathcal{A}_p^m} \left\{ \widehat{\mathcal{R}}_k(\mathbf{W})= \frac{1}{k} \sum_{i=1}^k \ell\left(X_{(i)},\mathbf{W}\right) \right\}
\end{equation}
%\begin{align*}
%\mathcal{L}_{\mathbf{W}}(V_i) = \min_{1 \leq j \leq m} \gamma(\mathbf{W},{V}_i,j)   = \min_{1 \leq j \leq m} \left(||V_i||_1 - \widetilde{V}_i^j\right).
%\end{align*}
%M\begin{equation}
%\mathbf{W}^\star \in \argmin_{\mathbf{W} \in \mathcal{A}_p^m} \frac{1}{n}\sum_{i = 1}^{n} \mathcal{L}_{\mathbf{W}}(V_i).
%\end{equation}
%Useful decomposition: for any $\mathbf{W} \in \mathcal{A}_p^m$
%\begin{align*}
%&\mathcal{R}(\widehat{\mathbf{W}}_n) - \mathcal{R}(\mathbf{W})
%= \\
%&\underbrace{\mathcal{R}(\widehat{\mathbf{W}}_n)-\mathcal{R}_n(\widehat{\mathbf{W}}_n)} +
%\underbrace{\mathcal{R}_n(\widehat{\mathbf{W}}_n) -\mathcal{R}_n(\mathbf{W})}_{\leq 0} + \underbrace{\mathcal{R}_n(\mathbf{W})- \mathcal{R}(\mathbf{W})}
%\\
%&\leq 2 \sup_{\mathbf{W} \in \mathcal{A}_p^m} |\mathcal{R}_n(\mathbf{W})- \mathcal{R}(\mathbf{W})|
%\end{align*}
%Taking the sup on the left side gives
%\begin{align*}
%\mathcal{R}(\widehat{\mathbf{W}}_n) - \mathcal{R}(\mathbf{W}^\star) \leq 2 \sup_{\mathbf{W} \in \mathcal{A}_p^m} |\mathcal{R}_n(\mathbf{W})- \mathcal{R}(\mathbf{W})|
%\end{align*}
Note that $\mathcal{A}_p^m$ is a closed and bounded set hence compact \citep{bourbaki2007topologie} thus there exists at least one solution which can be reached. The minimization problem of Eq. \eqref{original_problem} can be rewritten as
\begin{equation*}
\label{rephrased_problem}
\widehat{\mathbf{W}}_k \in \argmax_{\mathbf{W} \in \mathcal{A}_p^m} \frac{1}{k} \sum_{i = 1}^{k} g_{\mathbf{W}}(X_{(i)}).
\end{equation*}
The index of the column representing a good mixture can be defined with the mapping
\begin{align*}
\varphi : \llbracket 1,k \rrbracket \to \llbracket 1,m \rrbracket, \quad \varphi(i)= \argmax_{1 \leq j \leq m} \widetilde{X}_{(i)}^j
\end{align*}
%\begin{equation*}
%W^\star \in \arg\max_{W \in A} \sum_{i = 1}^{n} (VW)_{i}^{\varphi(i)},
%\end{equation*}
and the optimization problem becomes
\begin{equation}
\label{discrete_problem}
\argmax_{\mathbf{W} \in \mathcal{A}_p^m} \left\{ \frac{1}{k} \sum_{i = 1}^{k} (\mathbf{X} \mathbf{W})_{i}^{\varphi(i)} = \frac{1}{k}\sum_{i = 1}^{k} e_i (\mathbf{X}\mathbf{W}) e^{\varphi(i)} \right\}.
\end{equation}
%%%%%%%%%%%%%%%%%%%%%%%%%%%
\textbf{Illustrative example.} As a first go, consider the following example showing the way the matrix $\mathbf{W}$ recovers the different clusters. Assume that the vector $X \in \rset_+^p$ is exactly coming from a mixture of $m$ disjoint clusters $K_1,\ldots,K_m$ and for each sample $\mathbf{X}_i$, there exists $K_j$ such that $\|\mathbf{X}_i^{(K_j)}\|_1 = \|\mathbf{X}_i\|_1$. For all $j \in \llbracket 1,m \rrbracket$, denote $U^{j} \in [0,1]^p$ the uniform vector with support $K_j$, \textit{i.e.}, $U^j = (1/|K_j|)^{(K_j)}$. A solution to the optimization problem is given by any column-permutation of the matrix ${\mathbf{W}}$ whose columns are the vectors $U^j$. Indeed, the transformed data matrix is
$\widetilde{\mathbf{X}} = \mathbf{X}\mathbf{W}$ and for any sample $\mathbf{X}_i$ whose features are coming from a cluster $K_j$, we have
\begin{align*}
\forall l \neq j, \quad \widetilde{\mathbf{X}}_i^j = \mathbf{X}_i U^j = \mathbf{X}_i^{(K_j)} U^j \geq \mathbf{X}_i U^l = \widetilde{\mathbf{X}}_i^l.
\end{align*}
Taking $\varphi(i)= \argmax_{1 \leq l \leq m} \widetilde{\mathbf{X}}_i^l$ exactly recovers the cluster of index $j=\varphi(i)$. In the case where the large features of the different sample $\mathbf{X}_i$ are all equal, then the columns of the mixture matrix $\mathbf{W}$ tend exactly to uniform vectors with restricted support. %Now, if one of the large features is slightly bigger than the other then the associated column of $\mathbf{W}$ tends to a vertex of the simplex.

%%%%%%%%%%%%%%%%%%%%%%%%%%%
\subsection{Optimization on the Simplex}

\textbf{Problem relaxation.} One can directly solve the linear program \eqref{discrete_problem} but this formulation suffers from drawbacks. First, the solution could belong to a vertex of the simplex and would induce a unique direction. Second, it involves finding the mapping $\varphi$ among all the possible combinations which can be prohibited when $k$ or $p$ increases. Thus, one can solve a relaxed version of Eq. \eqref{discrete_problem} by introducing another matrix of mixtures $\mathbf{Z} \in \mathcal{A}_m^k$. The relaxed problem is
\begin{align}
\label{relaxed_problem}
(\widehat{\mathbf{W}}_k, \widehat{\mathbf{Z}}_k) &\in \argmax_{(\mathbf{W},\mathbf{Z}) \in  \mathcal{A}_p^m \times \mathcal{A}_m^k} \frac{1}{k} \sum_{i = 1}^{k}  \mathbf{X}_i \mathbf{W} \mathbf{Z}^i.
\end{align}
%%%%%%%%%%%%%%%%%%%%%%%%%%%%%
\textbf{Optimization problem.} We recognize the trace operator in Eq. \eqref{relaxed_problem} which is linear and can define an objective function $f: \mathcal{A}_p^m \times \mathcal{A}_m^k \to \rset$ that we need to maximize:
\begin{align*}
\label{def:obj_func}
\left\{
    \begin{array}{ll}
    	(\widehat{\mathbf{W}}_k, \widehat{\mathbf{Z}}_k) \in \argmax_{(\mathbf{W},\mathbf{Z})} f(\mathbf{W},\mathbf{Z}) \\
        f(\mathbf{W},\mathbf{Z}) = Tr(\mathbf{XWZ})/k
    \end{array}
\right.
\end{align*}
The objective function $f$ is bilinear in finite dimension hence continuous. Since maximization occurs on compact sets, there is at least one solution $(\widehat{\mathbf{W}}_k, \widehat{\mathbf{Z}}_k)$. However, it is not unique since any column-permutation of $\widehat{\mathbf{W}}_k$ along with the associated row-permutation of $\widehat{\mathbf{Z}}_k$ is also a valid solution. Indeed, any column (\textit{resp.} row) permutation consists of a multiplication on the right (\textit{resp.} left) side by a permutation matrix. For any $\sigma \in \mathfrak{S}_k$, consider the permutation matrix $\mathbf{P}_{\sigma} = (\delta_{i,\sigma(j)})_{1 \leq i,j \leq m}$. We have $\mathbf{P}_{\sigma}^T = \mathbf{P}_{\sigma^{-1}}$ so that
\begin{align*}
    (\widehat{\mathbf{W}}_k \mathbf{P}_\sigma)( \mathbf{P}_\sigma^T \widehat{\mathbf{Z}}_k) = \widehat{\mathbf{W}}_k (\mathbf{P}_\sigma \mathbf{P}_{\sigma^{-1}}) \widehat{\mathbf{Z}}_k= \widehat{\mathbf{W}}_k \widehat{\mathbf{Z}}_k.
\end{align*}
One may refer to \cite{meilua2006uniqueness,meilua2007comparing} for a discussion on the permutations of clustering solutions.

 %and references therein. \\
%%%%%%%%%%%%%%%%%%%%%%%%%%%%%%%%
\textbf{Regularization.}  The constraint of disjoint clusters can be satisfied by forcing the columns of the mixture matrix $\mathbf{W}$ to be orthogonal, \textit{i.e.}, for all $i < j, \langle W^i, W^j \rangle = 0$. This yields a penalized version of the objective function with a regularization parameter $\lambda > 0$
\begin{equation*} \label{def:reg_func}
\left\{
    \begin{array}{ll}
    	(\widehat{\mathbf{W}}_k, \widehat{\mathbf{Z}}_k) \in \argmax_{(\mathbf{W},\mathbf{Z})} f_{\lambda}(\mathbf{W},\mathbf{Z}) \\
        f_{\lambda}(\mathbf{W},\mathbf{Z}) = Tr(\mathbf{X}\mathbf{W}\mathbf{Z})/k - \lambda \sum_{i < j} \langle W^i, W^j \rangle
    \end{array}
\right.
\end{equation*}
with partial derivatives given by
\begin{equation*}
\left\{
    \begin{array}{ll}
    	\nabla_\mathbf{Z} f_{\lambda}(\mathbf{W}, \mathbf{Z})&= (\mathbf{X}\mathbf{W})^T/k \\
        \nabla_\mathbf{W} f_{\lambda}(\mathbf{W},\mathbf{Z})&= (\mathbf{Z} \mathbf{X})^T /k-\lambda \widetilde{\mathbf{W}}, \widetilde W^j = \sum_{i < j} W^i.
    \end{array}
\right.
\end{equation*}
%where $\widetilde W^j = \sum_{i < j} W^i.$

\textbf{Update rule.} \label{th1} The optimization problem can be addressed using an alternate scheme by computing projected gradient ascent at each iteration
\begin{equation} \label{eq:algorithm}
\left\{
    \begin{array}{ll}
        \mathbf{W}_{k+1} &= \Pi_{\mathcal{S}} \left(\mathbf{W}_k + \delta_k^{W} \nabla_\mathbf{W} f_{\lambda}(\mathbf{W}_k, \mathbf{Z}_k)\right) \\
		\mathbf{Z}_{k+1} &= \Pi_{\Delta_{m}} \left(\mathbf{Z}_k + \delta_k^{Z} \nabla_\mathbf{Z} f_{\lambda}(\mathbf{W}_{k+1}, \mathbf{Z}_k)\right)
    \end{array}
\right.
\end{equation}
where $\Pi_{\mathcal{S}}(\cdot),\Pi_{\triangle_m}(\cdot)$ are respectivetly the projection of each column onto a convex set $\mathcal{S} \subset \Delta_p$ and onto the probability simplex $\Delta_m$. The learning rates $\delta_k^{W},\delta_k^{Z}$ are step sizes found by backtracking line search. The convergence property of the optimization procedure is the same as the convergence of projected gradient descent as detailed in \citet{calamai1987projected,dunn1987convergence}.

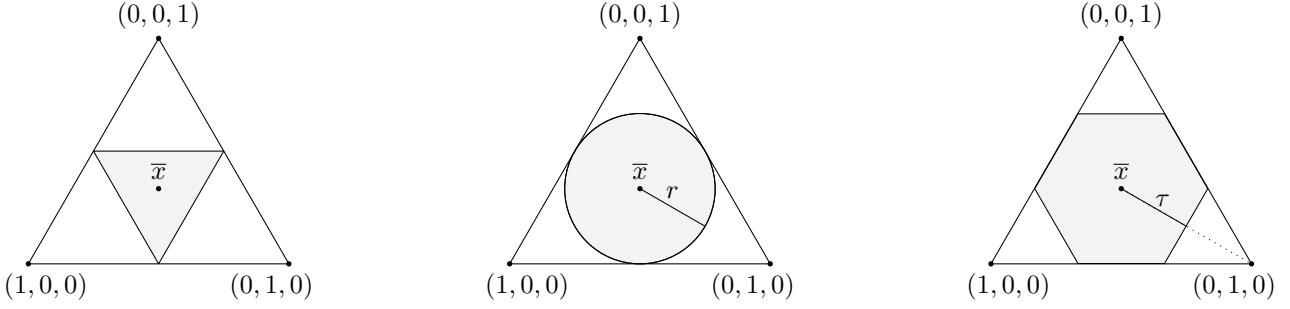
\begin{figure*}[t]
\centering
\begin{tikzpicture}[line join=round, scale=1.0]
]
% regular polygon with 3 sides. minimum size=diameter of circumcircle
\node[draw, regular polygon, regular polygon sides=3,
      minimum size=2\Radius, inner sep=0pt, outer sep=0pt] (tri) {};
\node[draw, regular polygon, shape border rotate=180, regular polygon sides=3,
      minimum size=\Radius, inner sep=0pt, outer sep=0pt] (tri) {};
\node[fill,opacity=0.05, regular polygon, shape border rotate=180, regular polygon sides=3,
      minimum size=\Radius, inner sep=0pt, outer sep=0pt] (tri) {};

%\filldraw[fill=mypurple] (outer.north) arc[start angle=90, end angle=210, radius=\Radius]
%         --cycle;
%\filldraw[fill=mypurple] (inner.30) arc[start angle=30, end angle=-90, radius=.5\Radius]
%         --(outer.-30)--cycle;
\draw (0,2) node[above] {$(0,0,1)$};
\draw (-1.5,-1.) node[below] {$(1,0,0)$};
\draw (1.5,-1.) node[below] {$(0,1,0)$};
\draw node[above,pos=.4] {$\ol x$};
\fill (0,0) circle (1pt);
\fill (0,2) circle (1pt);
\fill (-1.732,-1.) circle (1pt);
\fill (1.732,-1.) circle (1pt);
%\fill (outer.center) circle (1pt);
\end{tikzpicture}
\hfill
    \centering
    \begin{tikzpicture}[line join=round, scale=1.]

% circular nodes: minimum size=diameter
\node[draw, circle, minimum width=\Radius, inner sep=0pt, outer sep=0pt] (outer) {};
\node[draw, circle, minimum width=\Radius, inner sep=0pt, outer sep=0pt] (inner) {};
\node[fill,opacity=0.05, circle, minimum width=\Radius, inner sep=0pt, outer sep=0pt] (inner) {};

% regular polygon with 3 sides. minimum size=diameter of circumcircle
\node[draw, regular polygon, regular polygon sides=3,
      minimum size=2\Radius, inner sep=0pt, outer sep=0pt] (tri) {};

%\filldraw[fill=mypurple] (outer.north) arc[start angle=90, end angle=210, radius=\Radius]
%         --cycle;
%\filldraw[fill=mypurple] (inner.30) arc[start angle=30, end angle=-90, radius=.5\Radius]
%         --(outer.-30)--cycle;
\draw (0,2) node[above] {$(0,0,1)$};
\draw (-1.5,-1.) node[below] {$(1,0,0)$};
\draw (1.5,-1.) node[below] {$(0,1,0)$};
\draw node[above,pos=.4] {$\ol x$};
\draw (outer.center)--(inner.-30) node[midway,above]{$r$};
\fill (outer.center) circle (1pt);
\fill (0,2) circle (1pt);
\fill (-1.732,-1.) circle (1pt);
\fill (1.732,-1.) circle (1pt);
\end{tikzpicture}
\hfill
\centering
\begin{tikzpicture}[line join=round, scale=1.]
% regular polygon with 3 sides. minimum size=diameter of circumcircle
\node[draw, regular polygon, regular polygon sides=3,
      minimum size=2\Radius, inner sep=0pt, outer sep=0pt] (tri) {};
\node[draw, regular polygon, shape border rotate=180, regular polygon sides=6,
      minimum size=1.15\Radius, inner sep=0pt, outer sep=0pt] (tri) {};
\node[fill,opacity=0.05, regular polygon, shape border rotate=180, regular polygon sides=6,
      minimum size=1.15\Radius, inner sep=0pt, outer sep=0pt] (tri) {};

%\filldraw[fill=mypurple] (outer.north) arc[start angle=90, end angle=210, radius=\Radius]
%         --cycle;
%\filldraw[fill=mypurple] (inner.30) arc[start angle=30, end angle=-90, radius=.5\Radius]
%         --(outer.-30)--cycle;
\draw (0,2) node[above] {$(0,0,1)$};
\draw (-1.5,-1.) node[below] {$(1,0,0)$};
\draw (1.5,-1.) node[below] {$(0,1,0)$};
\draw node[above,pos=.4] {$\ol x$};
\draw (outer.center)--(inner.-30) node[midway,above]{};
\draw (0.55,-0.15) node  {$\tau$};
\draw [dotted](0,0) --(1.732,-1);
\fill (outer.center) circle (1pt);
\fill (0,2) circle (1pt);
\fill (-1.732,-1.) circle (1pt);
\fill (1.732,-1.) circle (1pt);
\end{tikzpicture}

    \caption{Simplex of $\rset^3$ with $\mathcal{S}_3^{\ell_1}$ (left), $\mathcal{S}_3^{\ell_2}$(center) and the $\mathbb{M}$-set $\mathcal{S}_3^{\tau}$(right).}
    \label{fig:simplex}
%%%%%%% COULD BE RELEVANT IN FUTURE FIG : https://tex.stackexchange.com/questions/505973/drawing-probabilities-on-a-simplex-in-tikz
\end{figure*}{}

\textbf{Projection step on $\mathcal{S}$.} In order to recover clusters that are not unit sets, we want to avoid the vertices of the simplex. Thus, we perform a projection step $\Pi_{\mathcal{S}}(\cdot)$ of each column of $\mathbf{W}$ onto a convex set $\mathcal{S}$. Several choices are to be considered, as illustrated in Figure \ref{fig:simplex}. Denote $\bar x = (1/p, \ldots, 1/p)$ the barycenter of the probability simplex $\Delta_p$ and consider the following manifolds: \\
\textit{(i) $\ell_1$ incircle:} the coordinate permutations of $(0,1/(p-1),\ldots,1/(p-1))$ are the centers of the faces of $\Delta_p$ and they define a reversed and scaled simplex $\mathcal{S}_p^{\ell_1}$. \\
\textit{(ii) $\ell_2$ incircle:} consider the euclidian ball $B_{2,p}(\bar x,r) = \{x \in \rset^p | \|x - \bar x\|_2 \leq r\}$. The radius value $r_p=1/\sqrt{p(p-1)}$ yields the $\ell_2$ inscribed ball of $\Delta_p$ along with $\mathcal{S}_p^{\ell_2} = \Delta_p \cap B_{2,p}(\bar x,r_p)$.  \\
\textit{(iii) $\mathbb{M}$-\textit{set}:} The previous manifolds do not scale well as the dimension grows and we shall discuss some theoretical results to see that their hypervolumes become very small. To escape from the curse of dimensionality, we consider the convex set where we cut off the vertices using a threshold $\tau$ of the distance $L = \|\bar x - e_j\|_2= \sqrt{(p-1)/p}$ between the barycenter and a vertex. It is also the intersection of the simplex $\Delta_p$ and an $\ell_{\infty}$ ball. We call this manifold the $\mathbb{M}$-\textit{set} $\mathcal{S}_p^{\tau}$ defined as
\begin{align*}
\mathcal{S}_p^{\tau} = \left\{x \in \Delta_p | \max_{1 \leq j \leq p} \left \langle x - \bar x, e_j - \bar x \right\rangle \leq \tau \|e_j - \bar x\|_2 \right\}.
\end{align*}
Define the radius $r_{\infty}^p(\tau) = 1 - (1-\tau)(p-1)/p$ then the $\mathbb{M}$-\textit{set} may be seen as the intersection of the simplex with a particular $\ell_{\infty}$ ball as
\begin{align*}
\mathcal{S}_p^{\tau} = \Delta_p \cap B_{\infty, p}\left(\bar x, \tau L\right) = \Delta_p \cap B_{\infty, p}\left(0, r_{\infty}^p(\tau) \right).
%\end{align*}
%or by translation
%\begin{align*}
\end{align*}
The projection onto the simplex is a well-studied subject \citep{daubechies2008accelerated,duchi2008efficient,chen2011projection,Condat2016}. For the projection onto the intersection of convex sets, one can perform a naive approach of alternate projections \citep{gubin1967method} or some refinements using the idea of Dykstra's algorithm \citep{dykstra1983algorithm,boyle1986method,bregman2003finding}.

\subsection{MEXICO Algorithm}
Starting from random matrices $(\mathbf{W}_0,\mathbf{Z}_0) \in \mathcal{A}_p^m \times \mathcal{A}_m^k$, the update rule of Eq.~\eqref{eq:algorithm} returns a pair of matrices $(\mathbf{W}_{mex},\mathbf{Z}_{mex})$ that are of great interest to analyze the dependence structure of the most extreme data and thus the support of extremes. On the one hand, the mixture matrix $\mathbf{W}_{mex}$ gives insights about the different clusters of features that are large simultaneously. On the other hand, the matrix $\mathbf{Z}_{mex}$ gives information about the probability of belonging to each cluster. %Indeed, those matrices are trained on the data matrix $\mathbf{X}$ so that
Each column $\mathbf{W}_{mex}^j$ represents a cluster $K_j$ and for each sample $\mathbf{X}_i, i \in \llbracket 1,k \rrbracket$, the $j^{th}$-row of the column $\mathbf{Z}_{mex}^i$ is the confidence for $X_i$ to belong to the cluster $K_j$.

A detailed pseudo-code of MEXICO is provided below in Algorithm \ref{alg:Mexico}. Since the margins of the data may be unknown, one could work with $\widehat T$ which is the empirical counterpart of the Pareto standardization $T$ as detailed in Appendix \ref{appendix:Pareto_scaling}. The output of the algorithm may be used for feature clustering (FC) or anomaly detection (AD) tasks. %The step 4 below may be withdrawn as discussed in Appendix \ref{appendix:MExico_no_normalization}.
\begin{algorithm}[!h]
%  \label{algo:Mexico}
\caption{MEXICO algorithm}
%\algsetup{linenodelimiter=.}
\begin{algorithmic}[1]
\REQUIRE Training data $(X_1,\ldots,X_n), 0<m<p, \lambda>0$ and rank $k(=\lfloor n \gamma\rfloor)$.
\STATE Initialize $(\mathbf{W}_0,\mathbf{Z}_0) \in \mathcal{A}_p^m \times \mathcal{A}_m^n$.
%\STATE Compute the index of extreme samples $\mathcal{I} =\{ i \in  \llbracket 1,n \rrbracket, ||\widehat{V}_i|| \geq n/k \}$.
\STATE Standardize the data $\widehat V_{(i)} = \widehat T(X_{(i)})$ (see Remark \ref{rmk:ParetoStandardization}).% \geq \ldots \geq  \|X_{(n)}\|$
\STATE Sort training data by decreasing order of magnitude
\qquad~ \qquad~\ $ \|\widehat V_{(1)}\|_\infty \geq \ldots \geq  \|\widehat V_{(n)}\|_\infty$.
%\STATE Compute the index of extreme samples $\mathcal{I} =\{ i \in  \llbracket 1,n \rrbracket, ||\widehat{V}_i|| \geq n/k \}$.
\STATE Consider the set of $k$ extreme training data  $\widehat V_{(1)}, \ldots, \widehat V_{(k)}$.
%\STATE Normalize extremes $\Theta = \big\{\widehat V_{(i)}\big / \|\widehat V_{(i)}\|_\infty \}_{(i \leq k)}$. %\in \rset^{n \times p}_+$ .
\STATE Compute $(\mathbf{W}_{mex},\mathbf{Z}_{mex}) \in \argmax_{(\mathbf{W},\mathbf{Z})} f_{\lambda}(\mathbf{W},\mathbf{Z})$ using update rule  \eqref{eq:algorithm}.
\STATE Given a new input ${X}_{\text{new}}$ standardized as $\widehat V_{\text{new}}$ with $\|\widehat V_{\text{new}}\|_\infty \geq \|\widehat V_{(k)}\|_\infty$, compute $\widetilde V_{\text{new}}=\widehat V_{\text{new}} \mathbf{W}_{mex}$.
\STATE Compute predicted cluster $\varphi_0 = \argmax_{1 \leq j \leq m}{\widetilde V_{\text{new}}^j}$.
\STATE (\textbf{FC}) Return cluster $\varphi_0$. \\ (\textbf{AD}) Return score $\ell(\widetilde V_{\text{new}},\mathbf{W}_{{mex}})$.
\end{algorithmic}
\label{alg:Mexico}
\end{algorithm}

\section{Theoretical Study} \label{sec:th_results}
This section provides some theoretical results. First, a theoretical analysis of the $\mathbb{M}$-\textit{set} is established. In order to compare the different manifolds of the previous section, we analyze the volume reduction performed in each case. Second, a non-asymptotic bound for the excess risk is detailed.

\begin{theorem}[Volume and ratio]\label{thm:thm2} Consider the probability simplex $\Delta_p$ and the different manifolds $\mathcal{S}_p^{\ell_1}, \mathcal{S}_p^{\ell_2}, \mathcal{S}_p^{\tau}$. For any bounded set $\mathcal{D} \subset \rset^p$, define its hypervolume $\mathcal{V}ol (\mathcal{D})$ and its ratio $\rho(\mathcal{D})$ as
\begin{align*}
\mathcal{V}ol (\mathcal{D}) = \int_{\rset^p} \un_{\mathcal{D}}(x) \mathrm{d}x, \quad \rho(\mathcal{D}) = \mathcal{V}ol (\mathcal{D}) / \mathcal{V}ol (\Delta_p).
\end{align*}
%We have the following hypervolumes
%\resizebox{\textwidth/2}{!}{
\begin{tabular}{|c|l|}
\hline
$\mathcal{S}$ & \multicolumn{1}{|c|}{Hypervolume $\mathcal{V}ol(\mathcal{S})$}  \\
\hline
 $\Delta_p$ & $\frac{\sqrt{p}}{\Gamma(p)}$ \\
\hline
$\mathcal{S}_p^{\ell_1}$ & $\frac{\sqrt{p}}{\Gamma(p)} \left(\frac{1}{(p-1)^{(p-1)}}\right)$  \\
\hline
 $\mathcal{S}_p^{\ell_2}$ & $\frac{\sqrt{p}}{\Gamma(p)} \left(\frac{\Gamma(p)}{\Gamma\left(\frac{p+1}{2}\right)} \frac{\pi^{(p-1)/2}}{\sqrt{p^p (p-1)^{(p-1)}}}\right)$  \\
\hline
$\mathcal{S}_p^{\tau}$ & $\frac{\sqrt{p}}{\Gamma(p)} \left(1-p(1-\tau)^{(p-1)}\left(\frac{p-1}{p}\right)^{(p-1)}\right)$  \\
\hline
\end{tabular}

The corresponding ratios are given by
\begin{align*}
\left\{
    \begin{array}{ll}
        \rho(\mathcal{S}_p^{\ell_1}) &= \frac{1}{(p-1)^{(p-1)}} \\
        \rho(\mathcal{S}_p^{\ell_2}) &= \frac{\Gamma(p)}{\Gamma\left(\frac{p+1}{2}\right)} \frac{\pi^{(p-1)/2}}{\sqrt{p^p (p-1)^{(p-1)}}} \\
        \rho(\mathcal{S}_p^{\tau}) &= 1-p\left[(1-\tau)\left(\frac{p-1}{p}\right)\right]^{(p-1)}
    \end{array}
\right.
\end{align*}
Moreover, when the dimension grows $p \to +\infty$ and for a fixed $\tau \in (0,1)$, we have $\rho(\mathcal{S}_p^{\ell_1}) \to 0$, $\rho(\mathcal{S}_p^{\ell_2}) \to 0$ and $\rho(\mathcal{S}_p^{\tau}) \to 1.$ Among studied subsets, in a high-dimensional setting the $\mathbb{M}$-\textit{set} is the only one not collapsing towards a unit set.
\end{theorem}

\begin{figure}[h]
\centering
\includegraphics[width=(0.5\textwidth)]{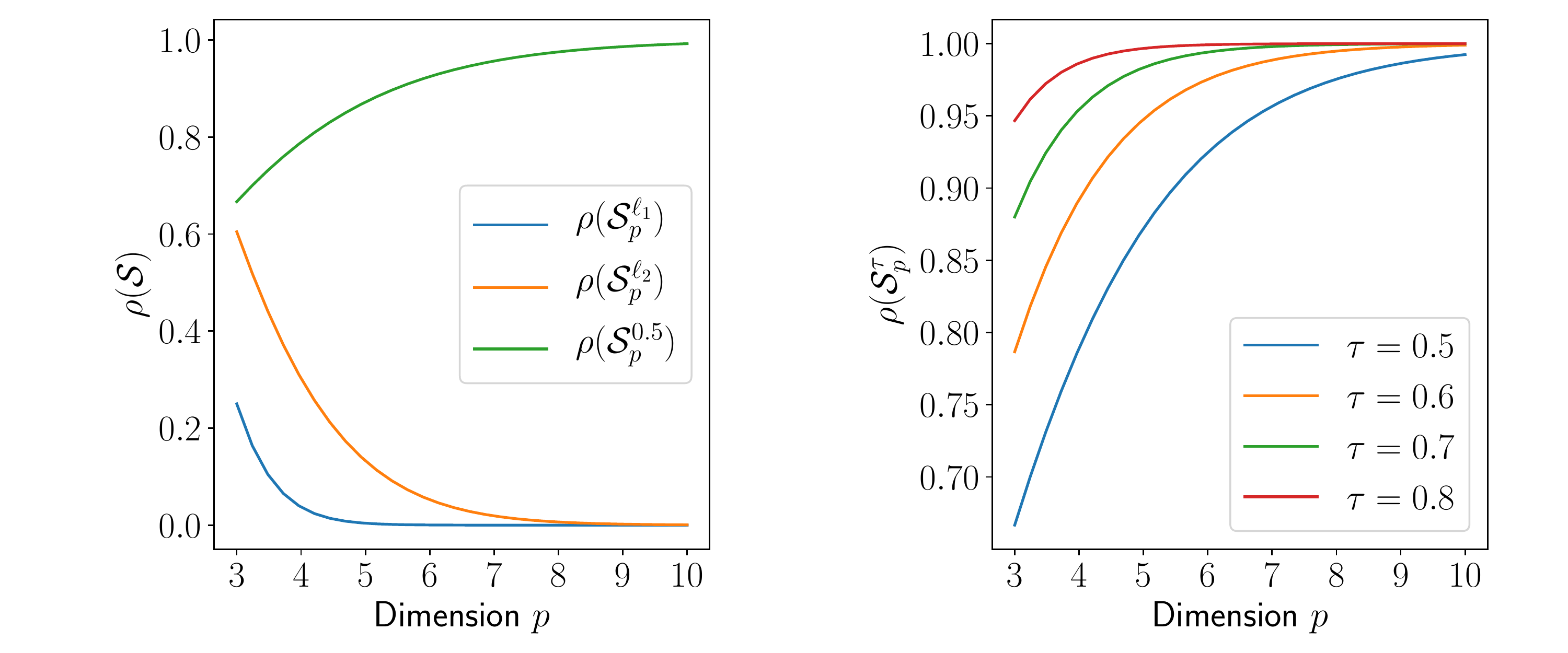}
\caption{Evolution of the  Ratio Volume with dimension $p$.}
\label{fig:ratio_ball}
\end{figure}

\begin{remark} (Selection of $\tau$)
With a high reduction of the probability simplex (\textit{i.e.} $\tau \to 0$), the vertices are avoided but discriminating the clusters is harder as the $\mathbb{M}$-\textit{set} tends to the barycenter of the simplex. This trade-off motivates the choice of the threshold $\tau$ and Figure \ref{fig:ratio_ball} shows the evolution of the ratios $\rho$ for the different manifolds.
\label{rem:sel_tau}
\end{remark}

The following proposition, whose proof is deferred to the appendix, shows that MEXICO algorithm can be seen as a contraction mapping.
\begin{proposition}[Lipschitz mapping]
  \label{prop:contractant} Given $x_1, x_2 \in \rset^p_+$ with norms greater than $t > 0$ %and normalized as $\theta_1, \theta_2 \in S_{||.||_\infty}$,
  the application
%$\theta \mapsto \theta\mathbf{W}_t^\star$ is $\frac{m}{2}$-lipschitz continuous \textit{i.e.}, %for two samples $\theta_1, \theta_2 \in \rset^p$
$x \mapsto x\mathbf{W}_t^\star$ is $\frac{m}{2}$-lipschitz continuous \textit{i.e.}, %for two samples $\theta_1, \theta_2 \in \rset^p$
\begin{align*}
  ||x_1 \mathbf{W}_t^\star - x_2 \mathbf{W}_t^\star ||_2 \leq \frac{m}{2}||x_1 - x_2 ||_2.
\end{align*}
Moreover, if $x_1$ and $x_2$ belong to the same feature cluster $K$ then
  $||x_1 \mathbf{W}_t^\star - x_2 \mathbf{W}_t^\star ||_2 \leq \frac{1}{2}||x_1 - x_2 ||_2.$ Hence, the transformation induced by MEXICO can be considered as a cluster contraction mapping \cite{boyd1969nonlinear}.
\label{prop:contraction}
\end{proposition}

Finally, the convergence rate of our method can be analyzed through the non-asymptotic bound on the risks \eqref{eq:def_true_risk} and \eqref{eq:def_emp_risk}. The following Theorem provides an upper bound for the excess risk. The proof is given in the appendix.
\begin{theorem}[Non-asymptotic bound]
  \label{th3}
  Consider the risk $\mathcal{R}_{t_\gamma}$ defined in \eqref{eq:def_true_risk} associated to the loss function \eqref{eq:loss_angular} computed on normalized data. Recall that $k = \lfloor n \gamma \rfloor$ and denote by $\mathbf{W}_{mex}$ the mixture matrix obtained by MEXICO. Then for $\delta \in (0,1)$, $n \geq 1$ and $\tau \leq 1$ we have with probability at least $1 - \delta$,
% VERSION EN 3 lignes ICI:
%\begin{align*}
%&\mathcal{R}_{t_\gamma}(\mathbf{W}_{t_\gamma}^\star) - \mathcal{R}_{t_\gamma}(\mathbf{W}_{mex}) \leq\\
%&\frac{1}{\sqrt{k}} 8\sqrt{2(1-\gamma)\log(4/\delta)} + \frac{4}{k} \bigg(\frac{4}{3} + \log (4/\delta) + \ldots \\
%&\qquad \qquad \ldots 2 \sqrt{ 2(1-\gamma) \log(4/\delta)}  + 2  \bigg)  + (1 - \tau)\left( \frac{p-1}{p} \right)
%\end{align*}
% VERSION EN 2 lignes ICI:
\begin{align*}
&\mathcal{R}_{t_\gamma}(\mathbf{W}_{mex})-\mathcal{R}_{t_\gamma}(\mathbf{W}_{t_\gamma}^\star) \leq \frac{1}{\sqrt{k}} 8\sqrt{2(1-\gamma)\log(4/\delta)} \ + \\
&\frac{1}{k} \bigg(\frac{16}{3}\log (4/\delta) +  8 \sqrt{ 2(1-\gamma) \log(4/\delta)}  + 2  \bigg)  + 2r_{\infty}^p(\tau).
\end{align*}

\end{theorem}
The upper bound stated above shows that the convergence rate is of order $O_\mathbb{P}(1/\sqrt{k})$ where $k$ is the actual size of the dataset required to estimate the support of extreme. This convergence rate matches the one of \citet{goix2016sparse}.
%\begin{corollary} Suppose that the assumptions of Theorem \ref{}. In addition, assume that the model optimization bias vanishes as $k \to \infty$,
%  \begin{align*}
%    &\mathcal{R}_{t_\gamma}(\mathbf{W}_t^\star) - \mathcal{R}_{t_\gamma}(\mathbf{W}_{mex})
%  \end{align*}
%\end{corollary}

%Refer to the appendix for the proof.

%%%%%%%%%%%%%%%%%%%%%%%%%%%%%%
\section{Numerical Experiments} \label{sec:exp}% \label{sec:applications}
We focus on popular machine learning tasks of \emph{feature clustering} and \emph{anomaly detection} to compare the performance of our algorithm against state-of-the-art methods for extreme events.
%To compare the performance of our algorithm against state-of-the-art methods, we focus on popular machine learning tasks for extreme events namely  and . %We consider various dimensions up to a big data framework where the dimension $p$ is relatively large compared to the number $n$ of samples.
Since the margins distributions of real-world data are unknown, the rank transformation as described in Remark \ref{rmk:ParetoStandardization} is considered. %Let us denote by $widehat{T}$ the empirical counterpart of $T$.
For ease of reproducibility, the code is available upon request.
\subsection{Feature Clustering}
Consider the feature clustering task where a new extreme sample $X_{\text{new}} \in \rset_+^p$ is to be analyzed. Since $X_{\text{new}}$ is extreme, our goal is to predict the features that are large simultaneously based on the dependence structure clusters, \textit{i.e.} the clusters given by MEXICO. For that matter, one can compute the transformed sample $\widetilde X_{\text{new}}=X_{\text{new}} \mathbf{W}_{mex} $ and assign the predicted cluster of features by $\text{Pred}(X_{\text{new}}) = \argmax_{1 \leq j \leq m}{\widetilde X_{\text{new}}^j}$.

Since MEXICO is an inductive clustering method, we focus on similar clustering algorithms namely spectral clustering \cite{ding2005equivalence} and spherical K-means \cite{janssen2020k}. \citet{janssen2020k} studied spherical K-means algorithm as a solution to perform clustering in extremes. We consider simulated data from an (asymmetric) logistic distribution where the dependence structure of extremes can be specified (see Appendix \ref{logistic_Appendix}). Given the ground truth class  samples, we leverage metrics using conditional entropy analysis: \citet{rosenberg2007v} define the following desirable objectives for any cluster assignment: Homogeneity (H), each cluster contains only members of a single class; Completeness (C), all members of a given class are assigned to the same cluster; v-Measure (v-M): the harmonic mean of Homogeneity and Completeness.%, which is actually equivalent to the mutual information.

The parameter setting is the following: dimension $p \in \{75, 100, 150, 200\}$, number of train samples $n_{\text{train}} = 1000$ and test samples $n_{\text{test}} = 100$. We use the metrics implemented by \textit{Scikit-Learn} \citep{sklearn2011}. The results, obtained over $100$ independently simulated dataset for each value of $p$, are gathered in Table \ref{tab:tab1}, where the values associated to MEXICO transcribe the best performance between projection method with Dykstra's algorithm and alternating projection. Both methods are detailed in the appendix. For each dimension $p$, bold characters indicate the best method when results are statistically significant using Mann-Whitney and Neyman-Pearson tests.

%\textbf{Feature Clustering (FC).}

%\section{Numerical Experiments} \label{sec:exp}

%\textbf{Anomaly Detection (AD).}

%%%%%%%%%%%%%%%%%%%%%%%%%%%%%%

\begin{table}[h]
\centering
%\resizebox{\textwidth/2}{!}{
\begin{tabular}{|c|ccc|}
\hline
$p$ & \multicolumn{3}{c|}{Spectral Clustering \citep{ding2005equivalence}}  \\
\hline
& H & C & v-M \\
\hline
75 & 0.925$\pm$ 0.054 & 0.937$\pm$0.040 & 0.931 $\pm$0.046  \\
\hline
100 & 0.918$\pm$0.058 & 0.934$\pm$0.039 & 0.926$\pm$0.048 \\
\hline
150 & 0.889$\pm$ 0.060 & 0.925$\pm$0.031 &  0.906$\pm$ 0.045 \\
\hline
200 & 0.886$\pm$0.047 & 0.928$\pm$0.024 & 0.906$\pm$0.034 \\ % ok merci
\hline
\hline
$p$ & \multicolumn{3}{c|}{Spherical-Kmeans \citep{janssen2020k}}  \\
\hline
& H & C & v-M \\
\hline
75 & 0.950$\pm$0.034 & 0.972$\pm$0.024 & 0.961$\pm$0.027  \\
\hline
100 & 0.943$\pm$0.031 & 0.967$\pm$0.024 & 0.955$\pm$0.026 \\
\hline
150 & 0.940$\pm$ 0.026 & 0.962$\pm$0.020 & 0.951$\pm$0.022 \\
\hline
200 & 0.940$\pm$0.018 & 0.962$\pm$0.014 & 0.951$\pm$0.015 \\
\hline
\hline
$p$ & \multicolumn{3}{c|}{MEXICO} \\
\hline
& H & C & v-M \\
\hline
75 & \textbf{0.978}$\pm$0.025 & 0.976$\pm$0.024 & \textbf{0.977}$\pm$0.024 \\
\hline
100 & \textbf{0.978}$\pm$0.020 & \textbf{0.979}$\pm$0.021 & \textbf{0.978}$\pm$0.020 \\
\hline
150 & \textbf{0.976}$\pm$0.015 & \textbf{0.980}$\pm$0.013 & \textbf{0.978}$\pm$0.014 \\
\hline
200 & \textbf{0.970}$\pm$0.015 & \textbf{0.975}$\pm$0.012 & \textbf{0.972}$\pm$0.013 \\
\hline
\end{tabular}
%}
\caption{Comparison of Homogeneity (H), Completeness (C) and v-Measure (v-M) on Simulated Data.}
\label{tab:tab1}
\end{table}

\subsection{Anomaly Detection}

%Consider the anomaly detection framework.
To predict whether a new extreme sample $X_{\text{new}} \in \rset_+^p$ is an anomaly, one may use the value of the loss function $\ell(X_{\text{new}},\mathbf{W}_{\text{mex}})$ as an anomaly score. If it is small then  the dependence structure of $X_{\text{new}}$ is well captured by the mixture $\mathbf{W}_{{mex}}$ and the behavior is rather \emph{normal}. Conversely, a high value means that $X_{\text{new}}$ cannot be approximated by a mixture of $\mathbf{W}_{mex}$ \textit{i.e.} it is more likely to be an outlier. The behavior of the extreme sample $X_{\text{new}}$ can be predicted using any decreasing function of the loss function $\ell$. In the experiment we use the inverse of the loss though one could consider the opposite of the loss as in \cite{goix2016sparse}.

We perform a comparison of three algorithms for anomaly detection in extreme regions: Isolation Forest \citep{liu2008isolation}, DAMEX \citep{goix2017sparse} and our method MEXICO. The algorithms are trained and tested on the same datasets, the test set being restricted to extreme regions. Five reference AD datasets are studied: shuttle, forestcover, http, SF and SA. Table \ref{tab_bonus} in the Appendix provides further dataset details. The experiments are performed in a semi-supervised framework where the training set consists of normal data only. More details about the preprocessing, model tuning and additional results are available in the appendix. The results of means and standard deviations are obtained over $100$ runs and summarized in  Table \ref{tab3}. Better performance are obtained with our anomaly detection approach compared to competing anomaly detection methods.

\iffalse
\begin{table}[h]
\centering
\begin{tabular}{|c|r|r|}
\hline
Dataset & \multicolumn{1}{c|}{Size} & \multicolumn{1}{c|}{Anomalies}   \\ %& $\tau$ & $\lambda$& Extremes
\hline
SA & 100 655 & 3377 (3.4\%) %& 7423
%& 0.7 & 5 \\
\\
\hline
SF & 73 237 & 3298 (4.5\%) %& 3258
%& 0.8 & 10 \\
\\
\hline
http & 58 725 & 2209 (3.8\%) %& 3049
%& 0.5 & 10 \\
\\
\hline
shuttle & 49 097 & 3511 (7.2\%) %& 4505
%& 0.7 & 5 \\
\\
\hline
forestcover & 286 048 & 2747 (0.9\%) %& 7458
%& 0.7 & 5 \\
\\
\hline
\end{tabular}
\caption{Datasets Description.}
\label{tab2}
\end{table}
\fi
\begin{table}[h]
\centering
%\resizebox{\textwidth/2}{!}{
\begin{tabular}{|c|cc|}
\hline
Dataset & ROC-AUC & AP  \\
\hline
& \multicolumn{2}{c|}{iForest \citep{liu2008isolation}} \\
\hline
SA & 0.886$\pm$0.032 & 0.879$\pm$0.031  \\
\hline
SF & 0.381$\pm$0.086 & 0.393$\pm$0.081  \\
\hline
http & 0.656$\pm$0.094 & 0.658$\pm$0.099 \\
\hline
shuttle & 0.970$\pm$0.020 & 0.826$\pm$0.055 \\
\hline
forestcover & 0.654$\pm$0.096 & 0.894$\pm$0.037 \\
\hline
& \multicolumn{2}{c|}{DAMEX \citep{goix2016sparse}} \\
\hline
SA & 0.982$\pm$0.002 & 0.938$\pm$0.012\\
\hline
SF & 0.710$\pm$0.031 & 0.650$\pm$0.034 \\
\hline
http & 0.996$\pm$0.002 & 0.968$\pm$0.009 \\
\hline
shuttle & 0.990$\pm$0.003 & 0.864$\pm$0.026 \\
\hline
forestcover & 0.762$\pm$0.008 & 0.893$\pm$0.010  \\
\hline
 &  \multicolumn{2}{c|}{MEXICO} \\
\hline
SA & 0.983$\pm$0.031 & \textbf{0.950}$\pm$0.011 \\
\hline
SF & \textbf{0.892}$\pm$0.013 & \textbf{0.812}$\pm$0.016 \\
\hline
http & 0.997$\pm$0.002 & \textbf{0.972}$\pm$0.012 \\
\hline
shuttle & 0.990$\pm$0.003 & 0.864$\pm$0.037 \\
\hline
forestcover & \textbf{0.863}$\pm$0.015 & \textbf{0.958}$\pm$0.006 \\
\hline
\end{tabular}
%}
\caption{Comparison of Area Under Curve of Receiver Operating Characteristic (ROC-AUC) and Average Precision (AP).}
\label{tab3}
\end{table}
%%%%%%%%%%%%%%%%%%%%%%%%%%%%%%
\section{Conclusion}\label{sec:conclusion}
Understanding the impact of shocks, \textit{i.e.}, extremely large input values on systems is of critical importance in diverse fields ranging from security or finance to environmental sciences and epidemiology. %From that perspective, clustering features of rare events is the cornerstone of many applications and may have dramatic consequences, \textit{e.g.}, a river dam failure in hydrological sciences or a miscarriage of justice when dealing with homeland security.
In this paper, we have developed a a rigorous methodological framework for clustering features in extreme regions, relying on the non-parametric theory of regularly varying random vectors. We illustrated our algorithm performance for both feature clustering and anomaly detection on simulated and real data. Our approach does not scan all the multiple possible subsets and outperforms existing algorithms. Future work will focus on the statistical properties %and guarantees
of the developed algorithm by further exploring links with kernel methods.
\newpage
%From a broader perspective, extreme data may have dramatic consequences and any clustering algorithm in such a setting should be used with great caution. Note that the purpose of MEXICO is to provide informative clusters of features although no guarantees on its robustness are provided so far. Finally, recovering clusters of data concerning sick patients at an early stage of a global pandemic is the key to slow down the resulting epidemic.

\bibliographystyle{icml2021}
\bibliography{main_arXiv.bbl}
\newpage
\appendix
\onecolumn

\begin{center}
%{\large {\bf\textsc{Supplementary material: Informative clusters for multivariate extreme }}}
{\large {\bf\textsc{Appendix: \\ Feature Clustering for Support Identification in Extreme Regions}}}
\end{center}

Section \ref{sec:proofs} gathers the proofs of theorems and propositions. Section \ref{sec:structure} details the probabilistic framework of EVT through Pareto standardization and logistic distributions. Section \ref{sec:self_sup} highlights links with self-supervised methods related to MEXICO. Section \ref{sec:details} is dedicated to numerical experiments: the preprocessing of the data, models selection and additional results. Further numerical experiments and variants are gathered in Section \ref{sec:Appendix_exp}.

\section{Proofs of Theorems \& Propositions} \label{sec:proofs}

\subsection{Proof of Proposition \ref{th0}}
\begin{proof} Proposition \ref{th0} can be derived from \citet{Engelke}[Theorem 2] and we provide another proof for the sake of clarity. $X$ is a standard regularly varing where with margin tails equivalent to Pareto distribution. Using the characterization of \citet{basrak2002characterization}, we have the following equivalence between the behavior of the vector and its components
\begin{align*}
\left(X \text{ is regularly varying}\right) \iff \left(\forall u \in \rset^p, \langle u, X \rangle \text{ is univariate regularly varying} \right)
\end{align*}
Each column $\widetilde{X}^{j}$ of $\widetilde X = X W$ is given by the linear combination $\widetilde{X}^{j} = \sum_{k = 1}^{p} X^{k} W_{k}^{j}$. Therefore, any linear combination of the form $\langle \widetilde u, \widetilde X \rangle$ with $\widetilde u \in \rset^m$ is actually a linear combination of the form $\langle u, X \rangle$. Indeed, we have for $\widetilde u \in \rset^m$,
\begin{align*}
\langle \widetilde u, \widetilde X \rangle = \sum_{j=1}^m \widetilde u_j \widetilde X^j = \sum_{j=1}^m \widetilde u_j \left(\sum_{k = 1}^{p} X^{k} W_{k}^{j} \right) = \sum_{k = 1}^{p} \left(\sum_{j=1}^m \widetilde u_j W_{k}^{j}\right) X^k.
\end{align*}
Since $X$ is regularly varying then any linear combination of the form $\langle \widetilde u, \widetilde X \rangle$ is univariate regularly varying, which exactly means, using the equivalence, that $\widetilde X$ is a regularly varying random vector. To find the tail index of the transformed vector, we rely on the following Lemma.

\begin{lemma} \label{lemma:pareto}
Let $X \in \rset_{+}^p$ be a regularly varying random vector with tail index $1$ and $\mathbf{W} \in \mathcal{A}_p^m$ a mixture matrix. Then the transformed vector $\widetilde{X} = X\mathbf{W}$ is regularly varying with tail index $\alpha = 1$.
\end{lemma}
%This follows from Lemma 3.3 in \citep{jessen2006regularly} and
\begin{proof}
Following Lemma 3.9 from \citet{jessen2006regularly}, let $\mathcal{A}_{W^j}$ denote the set $\{x, \langle W^j, x\rangle > 1\}$ where $W^j$ is the $j$-th column of $\mathbf{W}$, we want to show that $\mu(\mathcal{A}_{W^j}) > 0$.
Let $\epsilon = \max_{1\leq i \leq  p}W_i^j$. It follows that $\epsilon > 0$ otherwise $W^j = 0$ and it would not belong to $\Delta_p$. Let $i^* = \arg \max_{i\leq p} W_i^j$ \textit{i.e.} $\epsilon =  W_{i^*}^j$. As $ W_{i^*}^j$ and $V_{i^*}$ are positive, we have $\langle W^j, X\rangle \geq W_{i^*}^j X_{i^*} = \epsilon X_{i^*}.$ Therefore, $\{\epsilon X_{i^*} \geq t\} \subset \{\langle W^j, X \rangle \geq t\}$ and $t \PP{ \epsilon X_{i^*} \geq t} \leq t \PP{\langle W^j, X\rangle \geq t}$. By taking the limit on both sides of the equation, we obtain $0 < \epsilon \leq \mu(\mathcal{A}_{W^j})$ and conclude that each marginal of the transformed vector $\widetilde{X} = X\mathbf{W}$ is regularly varying with tail index $\alpha = 1$ thus $\widetilde X$ is regularly varying with tail index $1$. \qed

\end{proof}
Since the random vectors $X$ and $\widetilde X$ are regularly varying, we have the existence of nonzero Radon measures $\mu$ and $\widetilde \mu$ that are independent of the considered norm. Moreover, in virtue of Lemma \ref{lemma:pareto}, $\widetilde X$ is regularly varying with tail index $\alpha  = 1$, and considering the complementary of the unit sphere, defined by $\Omega_{m, ||.||_p}^c = \{x \in \rset_{+}^m,\|x\|_p > 1\}$. We have by definition
\begin{align*}
\widetilde {\mu}(\Omega_{m, ||.||_1}^c) =  \lim_{t \to \infty} t \PP{t^{-1}\widetilde V \in \Omega_{m,  ||.||_1}^c} = \lim_{t \to \infty} t \PP{||\widetilde V||_1 > t}.
\end{align*}
Using that $(1/m)||\widetilde V||_1 \leq ||\widetilde V||_\infty = \max_{1 \leq j \leq m} \widetilde V_j = \max_{1 \leq j \leq m} \left(\sum_{k = 1}^p V^k W_k^j \right)$ and $W_k^j \in [0,1]$, we have
\begin{align*}
\PP{||\widetilde V||_\infty > t} = \PP{\max_{1 \leq j \leq m} \widetilde V_j > t} \leq \PP{\sum_{k = 1}^p V^k > t}.
\end{align*}
We recognize the $\ell_1$-norm of the random vector $V \in \rset^p_+$ and obtain
\begin{align*}
\forall t>1, \quad t \PP{||\widetilde V||_1 > tm} \leq t \PP{||\widetilde V||_\infty > t}  \leq  t \PP{||V||_1 > t}.
\end{align*}
Taking the limit $t \to \infty$ on both sides provides the desired result $(1/m)\widetilde {\mu}(\Omega_{m, ||.||_1}^c) \leq \widetilde {\mu}(\Omega_m^c) \leq \mu(\Omega_p^c)$. \qed
\end{proof}

\subsection{Proof of Proposition \ref{prop:contractant}}
Let $x_1, x_2$  be two extreme samples %respectively normalized as $\theta_1$ and $\theta_2$
and $W^\star$ being a minimizer of $\mathcal{R}_{t}$, for the sake of simplicity we consider that $W^{\star^j}$ represents the cluster $K_j$ for $j \leq m$. Let $\tilde x_i = x_i W^\star \in \rset^m_+$ for $i \in \{1,2\}$. By definition,
\begin{align*}
  \| \tilde x_1 - \tilde x_2 \|_2^2
  = \|(x_1 - x_2)W^\star \|_2^2
  = \sum_{l = 1}^m \bigg(\big|\sum_{i=1}^p (x_1^i - x_2^i ) W_i^{\star^l} \big| \bigg)^2
  = \sum_{l = 1}^m \frac{1}{|K_l|^2} \left|\sum_{i\in K_l}(x_1^i - x_2^i)\right|^2
\end{align*}
which gives the upper bound
\begin{align*}
 \| \tilde x_1 - \tilde x_2 \|_2^2
 \leq \sum_{l = 1}^m \frac{1}{|K_l|^2} \sum_{i = 1}^p|x_1^i - x_2^i|^2|K_l|
  = \sum_{l = 1}^m \frac{1}{|K_l|} \sum_{i = 1}^p|x_1^i - x_2^i|^2
  \leq \frac{m}{|K_l|} \sum_{i = 1}^p|x_1^i - x_2^i|^2.
\end{align*}
As $\forall l \leq m, 2 \leq |K_l| \leq p$, $\theta \mapsto \theta W^\star$ is $(m/2)$-lipschitz continuous. Given the information that $x_1$ and $x_2$ belong to the same cluster $K_l$ the conclusion follows the same steps when considering solely the $l$-th column of $W^\star$ representing the cluster $K_l$.

\subsection{Proof of Theorem \ref{thm:thm2}}
First, recall the hypervolume of the $p$-simplex with side length $a$ and the hypervolume of the Euclidian ball of radius $R$ in dimension $p$: $\mathcal{V}ol(\Delta_p,a) = (\sqrt{p}/(p-1)!) \left(a/\sqrt{2}\right)^{p-1}$ and $\mathcal{V}ol\left(B_{2,p}(0,R)\right) = \pi^{p/2} R^p /\Gamma\left(\frac{p}{2} + 1\right).$

\textbf{Probability simplex $\Delta_p$.} The probability simplex we consider has a side length of $a=\sqrt{2}$ which gives the value of $\mathcal{V}ol(\Delta_p)$. \\
\textbf{$\ell_1$-incircle.} Regarding the $\ell_1$-ball, it is the scaled simplex whose side length is given by the distance between two face centers of $\Delta_p$. This length is equal to $\sqrt{2}/(p-1)$ and we deduce the volume $\mathcal{V}ol(\mathcal{S}_p^{\ell_1})$. \\
\textbf{$\ell_2$-incircle.} For the $\ell_2$-ball, denote $\mathcal{B} = (e_1,\ldots,e_p)$ the canonical basis and let $x \in \mathcal{S}_p^{\ell_2}, x = \sum_{i=1}^p \langle x,e_i \rangle e_i = \sum_{i=1}^p x_i e_i $. The vector $e_p^{'} = \sqrt{p} \bar x = (1/\sqrt{p},\ldots,1/\sqrt{p})$ is unitary and orthogonal to the simplex $\Delta_p$ with $\Delta_p \subset Span(e_p^{'})^\perp$. We have $\langle x, e_p^{'} \rangle = 0$ and we can complete the vector $e_p^{'}$ into an orthonormal basis $\mathcal{B'} = (e_1^{'},\ldots,e_p^{'})$ with $P = \mathcal{P}_{\mathcal{B},\mathcal{B'}}$ and $x = \sum_{i=1}^p \langle x, e_i \rangle e_i = \sum_{i=1}^{p-1} \langle x, e_i^{'} \rangle e_i^{'}.$ The hypervolume is invariant by translation so we make the projection of $\mathcal{S}_p^{\ell_2}$ onto $\rset^{p-1}$ to see that $\mathcal{V}ol(\mathcal{S}_p^{\ell_2}) = \mathcal{V}ol\left(B_{2,p-1}(0,r_p)\right)$ with $r_p = 1/\sqrt{p(p-1)}$ the radius of the $\ell_2$ inscribed ball of $\Delta_p$. This gives the value of $\mathcal{V}ol(\mathcal{S}_p^{\ell_2})$. \\
\textbf{$\mathbb{M}$-set.} Finally for the $\mathbb{M}$-set, we cut off with a threshold $\tau$ the length $L = \sqrt{(p-1)/p}$ between the barycenter $\bar x$ and a vertex $e_i$. We get $p$ smaller simplices and the volume we want is nothing but the difference between the volume of the simplex $\Delta_p$ and $p$ times the volume of a small simplex. To compute the hypervolume of one small simplex, we need to find its side length, knowing that its height is $(1-\tau)L$. We find a side length equal to $\sqrt{2}(1-\tau)(p-1)/p$ and can conclude for the value $\mathcal{V}ol(\mathcal{S}_p^{\tau})$.
\qed

We present in Figure \ref{fig3} the evolution of the ratio $\rho(\mathcal{S}^\tau)$ of the $\mathbb{M}$-set for different values of threshold $\tau$ and dimension $p$.

\begin{figure}[H]
  \centering
  \includegraphics[width=0.36\textwidth]{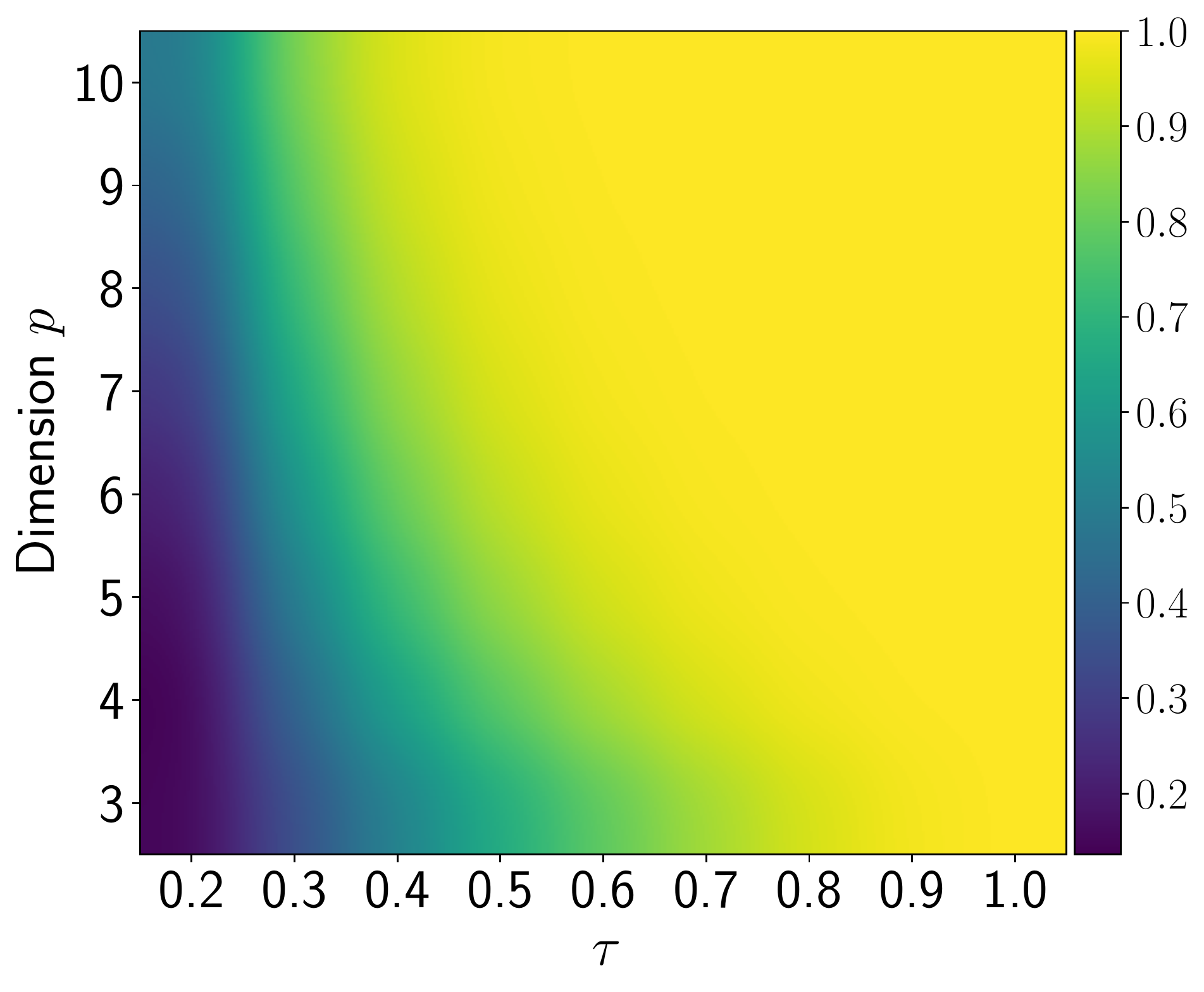}
  \caption{Evolution of $\rho(\mathcal{S}^\tau)$ with varying values of ($\tau,p$).}
  \label{fig3}
\end{figure}

\subsection{Proof of Theorem \ref{th3}}

  As stated in the introduction and in the preliminaries from Section~\ref{sec:EVT}, one can work with the angular measure $\Phi$ or the exponent measure $\mu$ in an \emph{equivalent} way when studying the limit distribution of extremes. In this proof, we focus on the angular components $\theta_i$ induced by the normalization $\theta_i = \Theta(X_i) = X_i / \|X_i\|_\infty$ so that the loss function is bounded in $[0,1]$. The proof relies on classical risk decompositions involving generalization and optimization errors.  We have
  \begin{align*}
&\mathcal{R}_{t_\gamma}(\mathbf{W}_{mex})-\mathcal{R}_{t_\gamma}(\mathbf{W}_{t_\gamma}^\star)  =
\underbrace{\left[\mathcal{R}_{t_\gamma}(\mathbf{W}_{mex}) - \widehat{\mathcal{R}}_k(\mathbf{W}_{mex}) \right]}_{generalization} +
\underbrace{\left[ \widehat{\mathcal{R}}_k(\mathbf{W}_{mex}) - \widehat{\mathcal{R}}_k(\widehat{\mathbf{W}_k}) \right]}_{optimization} +
\underbrace{\left[ \widehat{\mathcal{R}}_k(\widehat{\mathbf{W}_k}) - \mathcal{R}_{t_\gamma}(\widehat{\mathbf{W}_k})\right]}_{generalization} + \\
&\underbrace{\left[\mathcal{R}_{t_\gamma}(\widehat{\mathbf{W}_k}) -  \widehat{\mathcal{R}}_k(\widehat{\mathbf{W}_k}) \right]}_{generalization} +
\underbrace{\left[\widehat{\mathcal{R}}_k(\widehat{\mathbf{W}_k})- \widehat{\mathcal{R}}_k(\mathbf{W}_{t_\gamma}^\star) \right]}_{\leq 0} +
\underbrace{\left[\widehat{\mathcal{R}}_k(\mathbf{W}_{t_\gamma}^\star) - \mathcal{R}_{t_\gamma}(\mathbf{W}_{t_\gamma}^\star) \right]}_{generalization}.
\end{align*}
Therefore,
  \begin{align*}
\left|\mathcal{R}_{t_\gamma}(\mathbf{W}_{mex})-\mathcal{R}_{t_\gamma}(\mathbf{W}_{t_\gamma}^\star)\right|
&\leq \left| \widehat{\mathcal{R}}_k(\mathbf{W}_{mex}) - \widehat{\mathcal{R}}_k(\widehat{\mathbf{W}_k}) \right| +
\left|\mathcal{R}_{t_\gamma}(\mathbf{W}_{mex}) - \widehat{\mathcal{R}}_k(\mathbf{W}_{mex}) \right| + \\
&\left| \mathcal{R}_{t_\gamma}(\widehat{\mathbf{W}_k}) - \widehat{\mathcal{R}}_k(\widehat{\mathbf{W}_k})  \right| +
\left[\mathcal{R}_{t_\gamma}(\widehat{\mathbf{W}_k}) -  \widehat{\mathcal{R}}_k(\widehat{\mathbf{W}_k}) \right] +
\left| \mathcal{R}_{t_\gamma}(\mathbf{W}_{t_\gamma}^\star) - \widehat{\mathcal{R}}_k(\mathbf{W}_{t_\gamma}^\star)  \right|,
\end{align*}
Thus by taking the supremum over $\mathbf{W}$, on the latter term we finally obtain
    \begin{align}
    \label{eq:risk_decomposition}
    \mathcal{R}_{t_\gamma}(\mathbf{W}_{mex}) - \mathcal{R}_{t_\gamma}(\mathbf{W}^{\star}_{t_\gamma}) \leq 4 \sup_{\mathbf{W}}| \mathcal{R}_{t_\gamma}(\mathbf{W}) - \widehat{\mathcal{R}}_k(\mathbf{W}) | + | \widehat{\mathcal{R}}_k(\mathbf{W}_{mex}) - \widehat{\mathcal{R}}_k(\mathbf{\widehat{W}}_k) |
  \end{align}
The right-hand side of the Eq. \eqref{eq:risk_decomposition} is composed of two terms. The former, known as the generalization error, measures the gap between the true risk and its empirical counterpart whereas the latter is the optimization error between the solution $\mathbf{W}_{mex}$ found by MEXICO and the minimizer $\widehat{\mathbf{W}}_k$ of the empirical risk $\widehat{\mathcal{R}}_k$. The remainder of this proof relies on the following steps. We first provide a bound on the optimization error (Step 1) and then we upper-bound the generalization error (Step 2). This last term involves two quantities which are treated separately (Steps 2.1 and 2.2) on events with probability at least $(1-\delta/2)$. Collecting these two bounds and invoking the union bound concludes the proof (Step 3).

\textbf{Step 1 - Optimization error $|\widehat{\mathcal{R}}_k(\mathbf{W}_{mex}) - \widehat{\mathcal{R}}_k(\widehat{\mathbf{W}}_k) |$.}
Up to rescaling the $\mathbb{M}$-set, the optimization error is bounded as follows:
\begin{align*}
|\widehat{\mathcal{R}}_k(\mathbf{W}_{mex}) - \widehat{\mathcal{R}}_k(\widehat{\mathbf{W}}_k)| \leq 2 r_{\infty}^p(\tau).
\end{align*}
For the sake of simplicity, assume that the columns of $\mathbf{W}_{mex}$ correspond to the columns of $\widehat{\mathbf{W}}_k$ for any cluster $K_j$ with $j\leq m$. Up to permutation of the columns, the former assumption may be withdrawn. We have
\begin{align*}
    |\widehat{\mathcal{R}}_k(\mathbf{W}_{mex}) - \widehat{\mathcal{R}}_k(\widehat{\mathbf{W}}_k)|
     \leq \frac{1}{kp} \sum_{i=1}^k ||\theta_{(i)}||_1 ||(\mathbf{W}_{mex} - \widehat{\mathbf{W}}_k)^{\phi[(i)]}||_\infty
    \leq  \frac{1}{k}\sum_{i=1}^k ||(\mathbf{W}_{mex} - \widehat{\mathbf{W}}_k)^{\phi[(i)]}||_\infty
    \leq  2 r_{\infty}^p(\tau),
\end{align*}
where we used that both mixture matrices belong to the $\mathbb{M}$-set $\mathcal{S}_p^{\tau}$.

\textbf{Step 2 - Generalization error $\sup_{W}| \mathcal{R}_{t_\gamma}(\mathbf{W}) - \widehat{\mathcal{R}}_k(\mathbf{W}) |$.}
 % The statement of the theorem immediately derives from the uniform bounds on the deviations of $R_k$ on $W$ matrices stated in Theorem \ref{th4} below.
%\begin{theorem}
% \label{th4}
%  In the setting of theorem \ref{th3}, for all $\delta \in (0,1)$, we have with probability $1 - \delta$:
%  \begin{align}
%    \sup_{W}| \mathcal{R}_{t_\gamma}(\mathbf{W}) - \widehat{\mathcal{R}}_k(\mathbf{W}) | \leq
%  \end{align}
  %\begin{proof}[Proof of Theorem \ref{th4}]
Recall that $k = \lfloor n \gamma \rfloor$ along with the formula of the empirical risk $\widehat{\mathcal{R}}_k(\mathbf{W})$ and consider the surrogate empirical risk $\widetilde{R}_k(\mathbf{W})$ defined by:
    $$\widetilde{R}_k(\mathbf{W})= \frac{1}{k} \sum_{i=1}^k \ell(X_{(i)},\mathbf{W}), \qquad \widetilde{R}_k(\mathbf{W}) \defeq \frac{1}{k}\sum_{i=1}^n \ell(\theta_i,\mathbf{W})\un_{\{||X_i||_\infty \geq t_\gamma\}}$$
The generalization error may decomposed into
$$\mathcal{R}_{t_\gamma}(\mathbf{W}) - \widehat{\mathcal{R}}_k(\mathbf{W}) = \underbrace{\left[\mathcal{R}_{t_\gamma}(\mathbf{W}) - \widetilde{R}_k(\mathbf{W})\right]}_{A} + \underbrace{\left[ \widetilde{R}_k(\mathbf{W}) - \widehat{R}_k(\mathbf{W})\right]}_{B} $$
%$\widetilde{R_k}$ is not observed as $t_\gamma$ is unknown, it is an essential quantity in the following decomposition:
%$$ \sup_{W}| \mathcal{R}_{t_\gamma}(W) - \mathcal{R}_k(W) | \leq \underbrace{\sup_{W}| \mathcal{R}_{t_\gamma}(W) - \widetilde{\mathcal{R}_k}(W) |}_{A} + \underbrace{\sup_{W}| \widetilde{\mathcal{R}_k}(W) - \mathcal{R}_k(W) |}_{B}  $$
  %\end{proof}

%\begin{align*}
%\mathcal{R}_{t_\gamma}(\mathbf{W}) = \expec\left[\ell\left(\Theta(X),\mathbf{W})\right) \Big| \|X\|_\infty>t_\gamma\right], \quad \widehat{R_k}(\mathbf{W}) = \frac{1}{k}\sum_{i=1}^k \ell(\theta_i,\mathbf{W}), \quad    \widetilde{R_k}(\mathbf{W}) = \frac{1}{k}\sum_{i=1}^n \ell(\Theta(X_i),\mathbf{W})\un\{||X_i||_\infty \geq t_\gamma\}
%\end{align*}

%Denote by $A = \mathcal{R}_{t_\gamma}(\mathbf{W}) - \widetilde{R_k}(\mathbf{W})$ and $B = \widetilde{R_k}(\mathbf{W}) - \widehat{R_k}(\mathbf{W})$
%\end{step}
\textbf{Step 2.1 - Bound on A.}
%\textbf{.}
The first term is bounded as follows
\begin{align*}
A
&= \mathcal{R}_{t_\gamma}(\mathbf{W})- \widetilde{R}_k(\mathbf{W}) \\
&= \expec\left[\ell(\Theta(X), W) \mid ||X||_\infty  \geq t_\gamma \right] - \frac{1}{k}\sum_{i=1}^n \ell(\theta_i,\mathbf{W})\un_{\{||X_i||_\infty \geq t_\gamma\}} \\
&= \frac{1}{\gamma}\expec\left[\ell(\Theta(X), W) \un_{\{||X||_\infty  \geq t_\gamma\}}\right] -  \frac{1}{k}\sum_{i=1}^n \ell(\theta_i,\mathbf{W})\un_{\{||X_i||_\infty \geq t_\gamma\}}  \\
&= \left(\frac{1}{\gamma}-\frac{n}{k}\right)\expec\left[\ell(\Theta(X), W) \un_{\{||X||_\infty  \geq t_\gamma\}}\right] - \frac{1}{k}\sum_{i=1}^n \left( \ell(\theta_i,\mathbf{W})\un_{\{||X_i||_\infty \geq t_\gamma\}} - \expec\left[\ell(\Theta(X), W) \un_{\{||X||_\infty  \geq t_\gamma\}}\right] \right)
\end{align*}

Using that the loss $\ell(\cdot)$ is upper bounded by $1$ and $\expec\left[\un_{\{||X||_\infty  \geq t_\gamma\}}\right] = \mathbb{P}(||X||_\infty  \geq t_\gamma) = \gamma$, we can bound the first term as
\begin{align*}
    \left|\left(\frac{1}{\gamma}-\frac{n}{k}\right)\expec\left[\ell(\Theta(X), W) \un_{\{||X||_\infty  \geq t_\gamma\}}\right]\right|
    \leq \left|\left(\frac{1}{\gamma}-\frac{n}{k}\right) \gamma \right|
    \leq  \frac{|k - n\gamma|}{k}
    \leq \frac{1}{k},
\end{align*}
Therefore we have
\begin{align*}
|A| = \left|\mathcal{R}_{t_\gamma}(\mathbf{W})- \widetilde{R}_k(\mathbf{W}) \right| \leq \frac{1}{k} + \frac{1}{k} \left| \sum_{i=1}^n \left( \ell(\theta_i,\mathbf{W})\un_{\{||X_i||_\infty \geq t_\gamma\}} - \expec\left[\ell(\Theta(X), W) \un_{\{||X||_\infty  \geq t_\gamma\}}\right] \right)\right|.
\end{align*}
We shall treat the last term with Bernstein inequality. Denote by $S_n = \sum_{i=1}^n U_i$ with $U_i = \ell(\theta_i,\mathbf{W})\un_{\{||X_i||_\infty \geq t_\gamma\}}$ and note that $\mathbb{E}[U_i] = \expec\left[\ell(\Theta(X), W) \un_{\{||X|| \geq t_\gamma\}}\right]$, $E_n = \sum_{i=1}^n \mathbb{E}[U_i]$, $V_n = \sum_{i=1}^n Var(U_i) \leq n \gamma (1-\gamma)$.

 Bernstein inequality implies, for $y >0$,
  $\PP{ |S_n - E_n| /k > y } \le 2\exp\{-( y^2k^2/2 ) / (n
  \gamma(1-\gamma) + y k/3) \} :=\delta$.  Solving the latter bound for
  $y$ yields
  $ y = \frac{1}{k} \left( \frac{1}{3}\log(2/\delta) + \sqrt{(1/3 \log
      (2/\delta))^2 + 2 n\gamma(1- \gamma) \log(2/\delta)  }\right) $. Simplifying the latter bound using that for $a,b>0$, $\sqrt{a+b} \le \sqrt a + \sqrt b$, and that $n\gamma \le k+1$, we obtain that  with probability
  $1-\delta$,
  \begin{equation}\label{eq:boundA}
    |A| \le  \sqrt{ \frac{2}{k} (1-\gamma) \log(2/\delta)} + \frac{1}{k}\left(
    \frac{2}{3} \log (2/\delta) +  \sqrt{ 2(1-\gamma) \log(2/\delta)}  +1 \right)
  \end{equation}

\textbf{Step 2.2 - Bound on B.}
The second term is bounded as follows
%\textbf{.}
\begin{align*}
B &=
\widetilde{R}_k(\mathbf{W}) - \widehat{R}_k(\mathbf{W}) \\
&= \frac{1}{k}\sum_{i=1}^n \ell(\theta_i,\mathbf{W})\un_{\{\|X_i\|_\infty \geq t_\gamma\}} - \frac{1}{k}\sum_{i=1}^n \ell(\theta_i,\mathbf{W})\un_{\{\|X_i\|_\infty \geq \|X_{(k)}\|_\infty\}} \\
&= \frac{1}{k}\sum_{i=1}^n \ell(\theta_i,\mathbf{W}) \left[\un_{\{\|X_i\|_\infty \geq t_\gamma\}} - \un_{\{\|X_i\|_\infty \geq \|X_{(k)}\|_\infty\}} \right]
\end{align*}
Again, using that the loss $\ell(\cdot)$ is upper bounded by $1$ and by triangle inequality we get
\begin{align*}
|B|
&\leq \frac{1}{k}\sum_{i=1}^n  \left|\un_{\{\|X_i\|_\infty \geq t_\gamma\}} - \un_{\{\|X_i\|_\infty \geq \|X_{(k)}\|_\infty\}} \right|
\end{align*}
To analyze this term, we shall consider whether $\|X_{(k)}\|_\infty  \geq t_\gamma$ or not. We have
\begin{align*}
\sum_{i=1}^n  \left|\un_{\{\|X_i\|_\infty \geq t_\gamma\}} - \un_{\{\|X_i\|_\infty \geq \|X_{(k)}\|_\infty\}} \right|
& =
      \begin{cases}
        \sum _{i=k+1}^n \un_{\{ \|X_{(i)}\|_\infty \geq t_\gamma \}}  & \text{ if } \|X_{(k)}\|_\infty \geq t_\gamma \\
         \sum _{i=1}^k  \un_{\{ \|X_{(i)}\|_\infty <  t_\gamma \}} & \text{ otherwise }
      \end{cases} \nonumber \\
    &=
      \begin{cases}
         \sum _{i=1}^n \un_{\{ \|X_{(i)}\|_\infty \geq t_\gamma \}} - \frac{k }{k}  & \text{ if } \|X_{(k)}\|_\infty \geq t_\gamma \\
         \sum _{i=1}^n  \un_{\{ \|X_{(i)}\|_\infty <  t_\gamma \}} - \frac{n-k}{k} & \text{ otherwise }
      \end{cases} \nonumber \\
      &  = \left| \sum_{i=1}^n \un_{\{\| X_i\|_\infty  \geq t_\gamma  \}} - 1\right|\nonumber
\end{align*}
Therefore
\begin{align*}
|B| \leq \frac{1}{k} \left| \sum_{i=1}^n \un_{\{\| X_i\|_\infty  \geq t_\gamma  \}} - 1\right|
\end{align*}
Denote by $S_n = \sum_{i=1}^n Z_i$ with $Z_i=\un_{\{\| X_i\|_\infty  \geq t_\gamma  \}}$ and observe that $\mathbb{E}[S_n] = n\gamma$. We finally have
\begin{align*}
    |B| \leq \frac{\left| S_n - n\gamma\right| }{k}+ \frac{1}{k}
\end{align*}
 Similarly to the bound on A, we obtain that  with probability
  $1-\delta$,
  \begin{equation}\label{eq:boundB}
    B \le  \sqrt{ \frac{2}{k} (1-\gamma) \log(2/\delta)} + \frac{1}{k}\left(
    \frac{2}{3} \log (2/\delta) +  \sqrt{ 2(1-\gamma) \log(2/\delta)}  +1 \right)
  \end{equation}

\textbf{Step 3 - Final upper-bound.}
Using the union bound and combining the inequalities \eqref{eq:boundA} and \eqref{eq:boundB} from Steps 2.1 and 2.2, we obtain that with probability at least $(1-\delta)$,
  \begin{align*}
\sup_{W}| \mathcal{R}_{t_\gamma}(\mathbf{W}) - \widehat{\mathcal{R}}_k(\mathbf{W}) | \leq \frac{1}{\sqrt{k}} 2\sqrt{2(1-\gamma)\log(4/\delta)}  + \frac{1}{k} \left(
    \frac{4}{3} \log (4/\delta) +  2 \sqrt{ 2(1-\gamma) \log(4/\delta)}  + 2  \right),
  \end{align*}
and plug the bound on the optimization error (Step 1) into the decompostion \eqref{eq:risk_decomposition} to conclude the proof. \qed
\section{Probabilistic Framework \& Dependence of Extremes}
\label{sec:structure}

\subsection{On the Pareto Standardization $T$ and its Empirical Counterpart $\widehat T$}
\label{appendix:Pareto_scaling}
As the components of a random vector are not necessarily on the same scale, componentwise standardisation is a natural and necessary preliminary step. The Pareto standardization $T$ (and its empirical counterpart $\widehat T$) is mentionned in Algorithm \ref{alg:Mexico}. Following common practice in multivariate extreme value analysis \cite{beirlant2006statistics}, the input data $(X_i)_{i \in \{1,\ldots, n\} }$ is standardised by applying the rank-transformation: %$\forall i \in \{1, \cdots,n\}$,
$$\widehat{T}(x) =\bigg(1 / \Big(1 - \widehat{F}_j(x) \Big) \bigg)_{j=1, \ldots, d} $$ for all $x = (x^{1}, \ldots, x^{p}) \in \rset^p$ where $\widehat{F}_j (x) \defeq \frac{1}{n+1}\sum_{i=1}^n \mathds{1}\{X_i^j \leq x\}$ is the $j^{th}$ empirical marginal distribution.
Denoting by $\widehat V_i$ the standardized variables, $\forall i \in \{1, \ldots, n \}, \widehat V_i = \widehat{T}(X_i)$.  The marginal distributions of $\widehat V_i$ are well approximated by  standard Pareto distribution. The approximation error comes from the fact that the empirical \emph{c.d.f}'s are used in $\widehat T$ instead of the genuine marginal \emph{c.d.f.}'s $F_j$. After this standardization step, the selected extreme samples are  $\{ \widehat V_i, \|\widehat V_i\|_\infty \geq  V_{(\lfloor n \gamma \rfloor)} \}$.

\subsection{Logistic distribution - An illustration of extremes dependence structure}
\label{logistic_Appendix}
%As mentioned in Section~\ref{sec:toyExemple}
 The logistic distribution with dependence  parameter $\delta  \in (0, 1]$  is defined in $\rset^p$ by its  \textit{c.d.f.}
$F(x) = \exp\big\{ - (\sum_{j=1}^p {x^{(j)}}^\frac{1}{\delta})^{\delta} \big\}$. It can be considered as a simplified counterpart of the asymmetric logistic.
Samples from both asymmetric logistic distribution and logstic distribution can be simulated according to algorithms proposed in  \citet{stephenson2003simulating}. \autoref{fig:Logistic_Exemples} illustrates the logistic with various values of $\delta$. As $\delta$ gets close to~$1$ extremes tend to occur in a non concomitant design, \emph{i.e.} the probability of a simultaneous excess of a high threshold by more than one vector component is negligible. Conversely, as of $\delta$ gets close to $0$, extreme values are more likely to occur simultaneously. %These two distinct tail dependence structures are respectively called  `asymptotic independence' and `asymptotic dependence' in the EVT terminology.
\begin{figure}[h]{}
  \begin{subfigure}[t]{0.32\textwidth}
    \includegraphics[width=\textwidth]{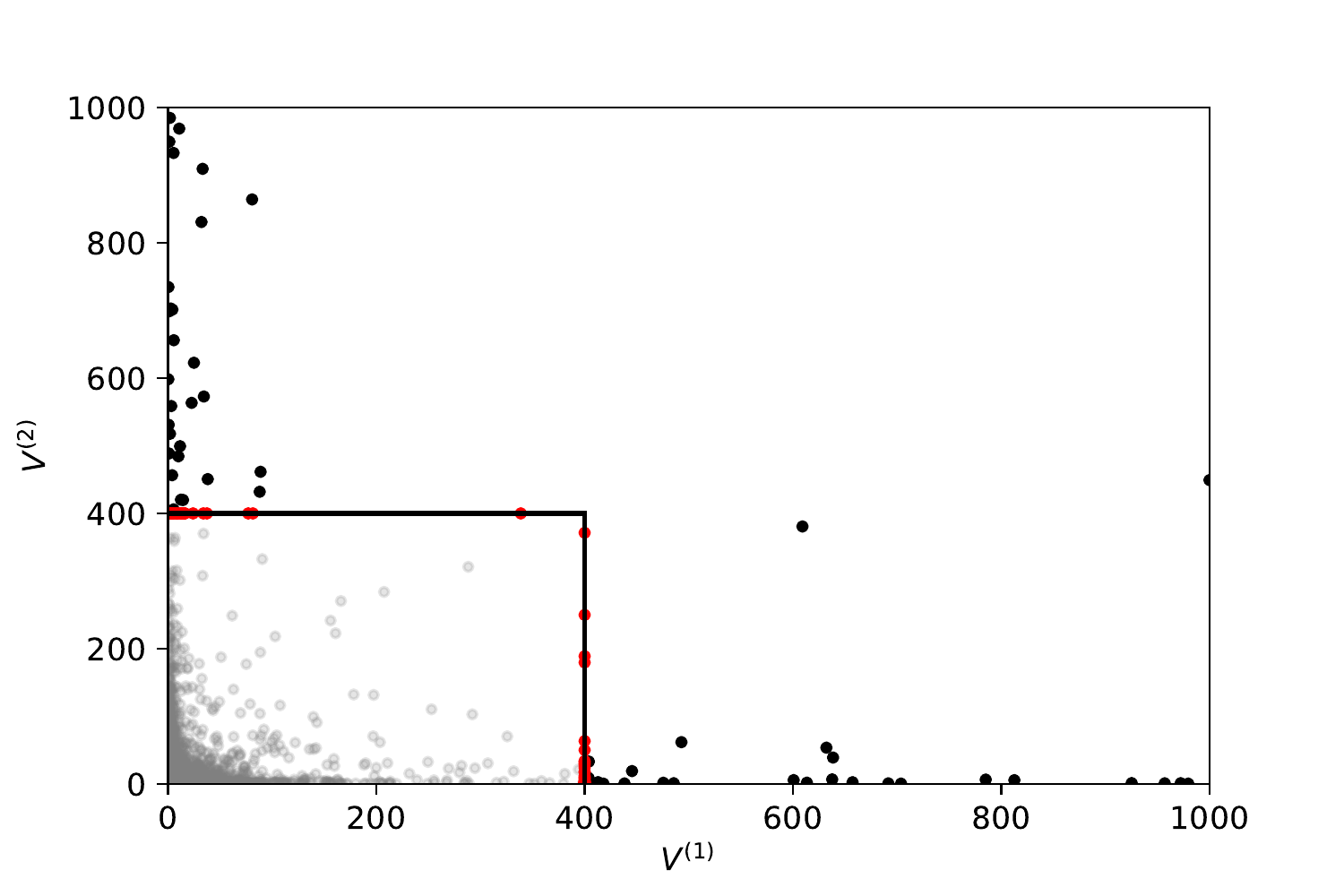}
    \caption{low tail dependence}
    \label{fig:logistic_alpha_09}
    \end{subfigure}
\begin{subfigure}[t]{0.32\textwidth}
    \includegraphics[width=\textwidth]{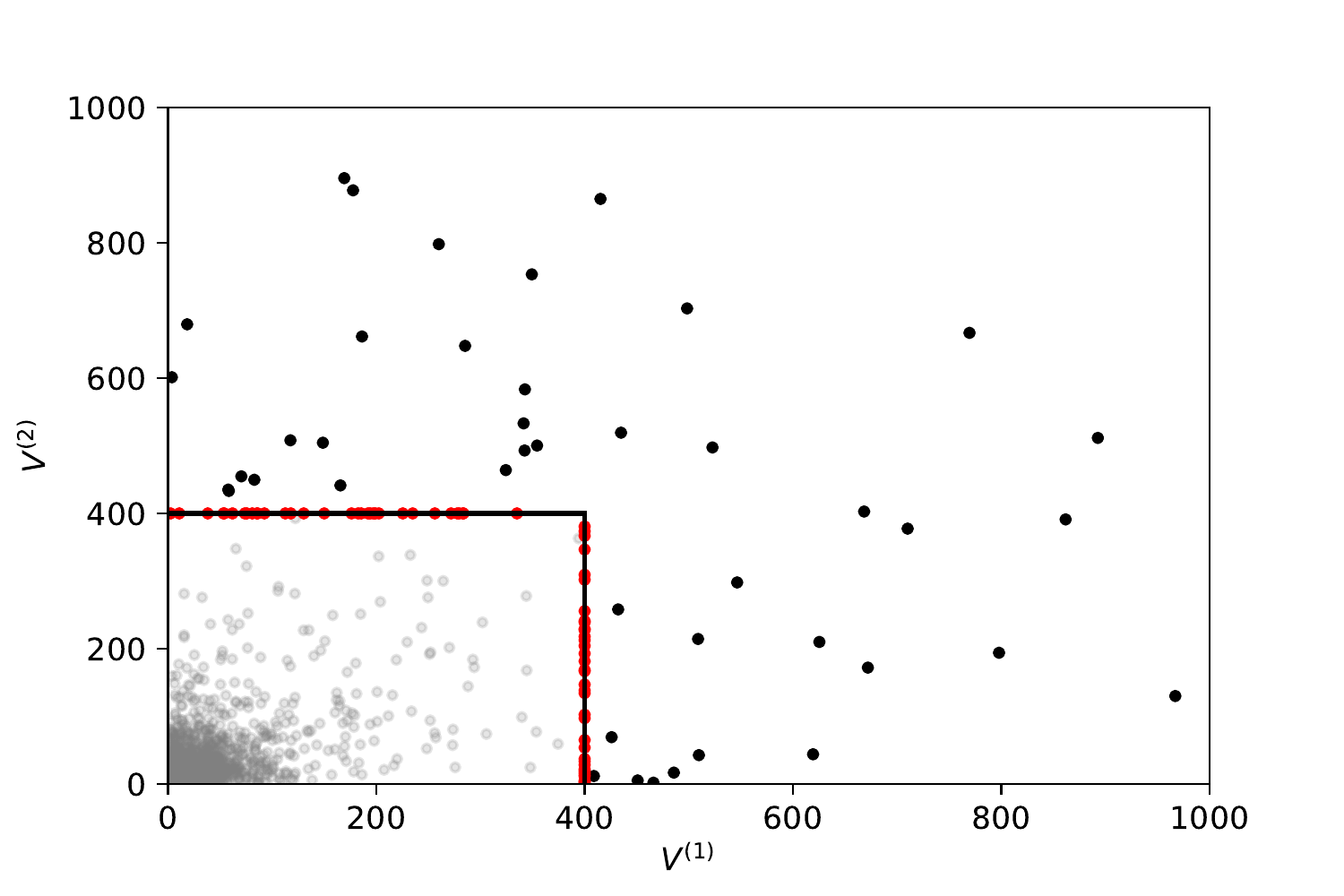}
    \caption{ moderate tail dependence}
    \label{fig:logistic_alpha_05}
        \end{subfigure}
\begin{subfigure}[t]{0.32\textwidth}
    \includegraphics[width=\textwidth]{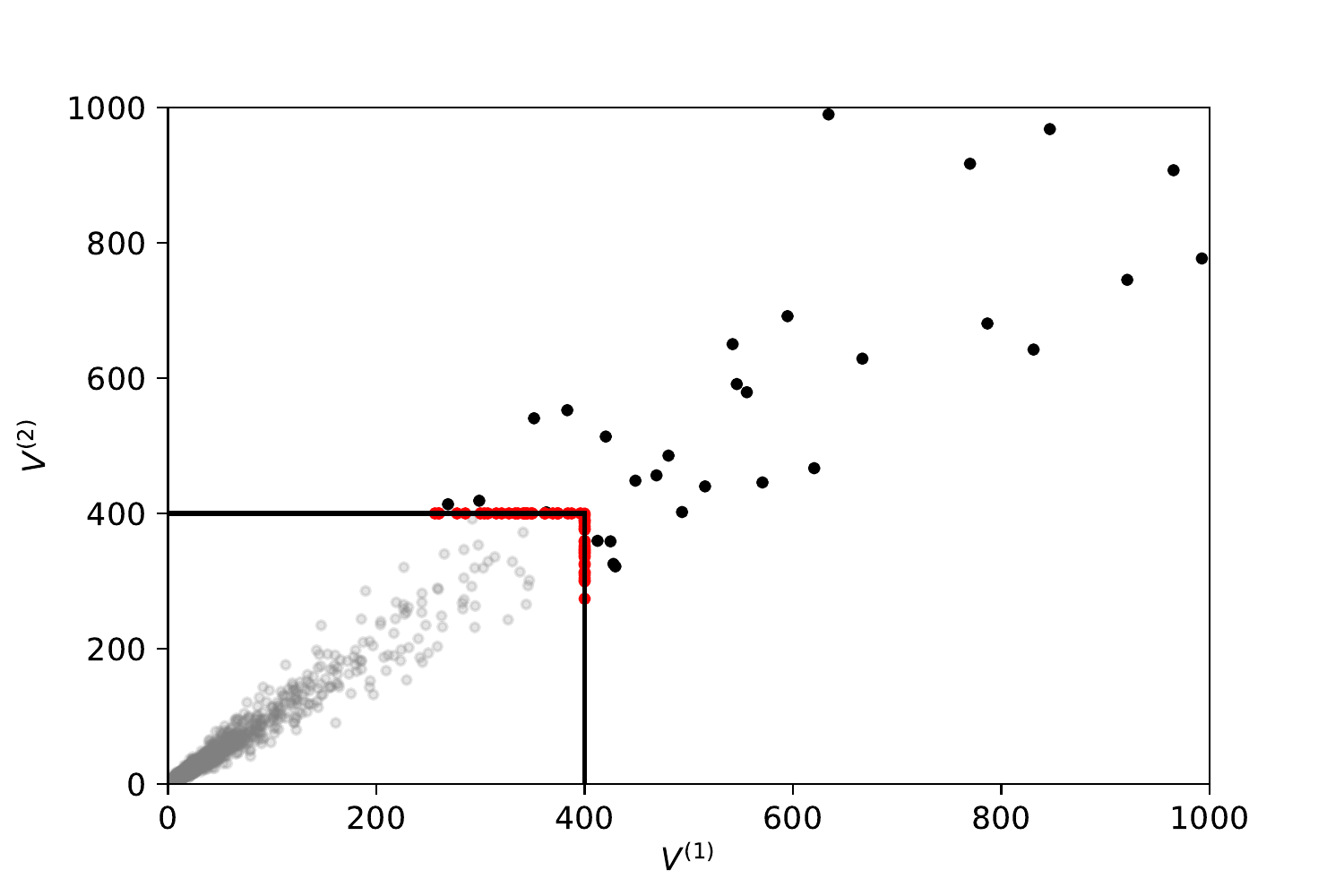}
    \caption{high tail dependence}
    \label{fig:logistic_alpha_01}
\end{subfigure}
\caption{Illustration of the distribution of the angle $\Theta(X)$ obtained with bivariate samples $X$ generated from a logistic model with different coefficients of dependence ranging from near asymptotic independence  \autoref{fig:logistic_alpha_09} ($\delta = 0.9$) to high asymptotic dependence \autoref{fig:logistic_alpha_01} ($\delta = 0.1$) including moderate dependence \autoref{fig:logistic_alpha_05} ($\delta = 0.5$). Non extreme samples are depicted in gray, extreme samples are represented with black dots and the angles  $\Theta(X)$ (extreme samples projected on the $\ell_\infty$ sphere) are plotted in red. Note that not all extremes are shown since the plot was truncated for a better visualization. However all projections on the sphere are shown.}
\label{fig:Logistic_Exemples}
\end{figure}

\subsection{Structure Dependence} A common approach for modeling extreme events is to use some flexible parametric subclass of distributions \citep{stephenson2009high}. Let $F$ denote a multivariate extreme value \textit{c.d.f}. Each univariate marginal distribution of $F$ is then a generalized extreme value distribution. More precisely, for $j \in \llbracket 1,p \rrbracket$, the $j$-th univariate marginal \textit{c.d.f} is given by
\begin{align*}
F^{j}\left(x^{j}\right)=\exp \left[-\left(1+\xi_{j}\left(x_{j}-\mu_{j}\right) / \sigma_{j}\right)_{+}^{-1 / \xi_{j}}\right]
\end{align*}
where $\mu_{j}, \xi_{j}$ and $\sigma_{j}>0$ are respectively the location, shape and scale parameters for the $j$-th marginal distribution.

The asymmetric logistic model provides perhaps the most popular parametric subclass of multivariate extreme value distributions. It is defined as follows.
\begin{definition}\label{def:asym}(Asymmetric logistic, \citep{Tawn90})
Define $y_{j}=-1 / \log F^{j}\left(x^{j}\right)$ for $j \in \llbracket 1,p \rrbracket$. The $p$-dimensional asymmetric logistic \textit{c.d.f} is
\begin{align*}
F\left(x\right)=\exp \left[-\sum_{K \subset \llbracket 1,p \rrbracket}\left\{\sum_{j \in K}\left(\beta_{j,K} y_{j}\right)^{1/\alpha_{K}}\right\}^{\alpha_{K}}\right],
\end{align*}
where $\alpha_{K}$ are the dependence parameters and $\beta_{j, K}$ are the asymmetry parameters.
\end{definition}

\section{Self-supervised Learning} \label{sec:self_sup}

The task tackled in this paper revolves around understanding the dependence
structure of extremes. At first glance, one may sum up our goal as a \textit{self-supervised learning} problem where
the objective would be to predict the $\ell_1$ norm
of an extreme sample relying on subgroups of features best contributing to these features being extreme. Following \citet{kolesnikov2019revisiting}, given an extreme input $x$ one could build a \textit{preformulated} label to be predicted as
$||x||_1$, a regular linear regression model could be trained on the resulting trivial
labeled dataset $\{(X_{(i)}, Y_{(i)})\}_{i=1}^k$ with $Y_{(i)} = ||X_{(i)}||_1$ and analyse the parameters of the linear regression model. Although, to the best of our
knowledge there is no linear regression models which directly deals with non-predefined groups of features
best contributing to the $\ell_1$ norm. For the sake of completeness we detail below the linear models involving such group analysis:
\begin{itemize}
  \item \textbf{Group Lasso} \cite{yuan2006model}\textbf{.}\ The fomulation of our problem of interest as a group lasso supposes that the $p$ predictors are divided into $L$ groups, where $p_l$ represents the number of features in group $l$. The self-supervised learning problem rewrites as
  \begin{align*}
    \min_{\beta_l \in \rset^p} \left\| \|X \un\|_1 - \sum_{l=1}^L X_{l}\beta_l \right\|_2^2 + \lambda \sum_{l=1}^L \sqrt{p_l}\|\beta_l\|_2,
  \end{align*}
where $\un = (1, \ldots, 1)^T$ denotes the column vector of size $k$ solely containing $1$ and $\lambda \in \rset$. However, the group of features $L$ in the solution are predefined before solving the optimization problem. In that aspect,  our framework differs from Group Lasso as we seek to find the groups of features.

  \item \textbf{Sparse Group Lasso} \cite{friedman2010note}\textbf{.}\ Sparse group lasso consists of a linear combination of group lasso and a lasso penalization that provide solutions that are both \textit{between} and \textit{within} group sparse. The minimization problem is the following
  \begin{align}
    \min_{\beta_l \in \rset^p} \left\| \|X \un\|_1 - \sum_{l=1}^L X_{l}\beta_l \right\|_2^2 + \tau \lambda \sum_{l=1}^L |\beta_l| + (1 - \tau)\lambda \sum_{l=1}^L \sqrt{p_l}\|\beta_l\|_2,
    \label{eq:sgl}
  \end{align}
where $\tau \in [0,1]$ balances the relative importance of sparsity term lasso or the group in the optimization problem. Although, in that setting the different groups of features still remain to be set in advance.
  \item \textbf{Adaptive Group Lasso}   \cite{wang2008note}\textbf{ \& \textbf{Adaptive Sparse Group Lasso}} \cite{mendez2020adaptive}\textbf{.} \
  Adaptive Group Lasso and Adaptive Sparse Group Lasso consist in setting a penalizations parameter $\lambda_l$ fot each features group $l\leq L$ in Eq. \ref{eq:sgl}. Once again, the $L$ groups of features must be predefined and do no fit our setting as these groups are unknown and considering the total number of combination ($2^p - 1$) could be computationally limiting as $p$ gets large.

\end{itemize}

\section{Numerical Experiments Details} \label{sec:details}

\subsection{Model Selection}\, DAMEX and Isolation Forest hyperparameters are the same as in \citet{goix2016sparse, goix2017sparse}. Note that the performance of DAMEX and Isolation Forest in Table \ref{tab6} differ from results from Table $4$ in \citet{goix2016sparse} since they report performance combining both the extreme and non-extreme regions: they rely on DAMEX for samples falling in the extreme regions and Isolation Forest on the non-extreme regions as they depict in their Figure 5. MEXICO parameters are set according to Remarks \ref{rem:sel_k} and \ref{rem:sel_tau}. Note that all performance reported in Tables \ref{tab6} are solely computed on test samples considered as extreme since we focus on extreme regions.

\subsection{Additional Results Feature Clustering}

We present the full results of the performance of MEXICO regarding the feature clustering task. The projection step is either performed using alternating projections based on the method POCS (Projection Onto Convex Sets) or with the more elaborate technique Dykstra.

\begin{table}[h]
\centering
\resizebox{\textwidth}{!}{
\begin{tabular}{|c|ccc|ccc|ccc|}
\hline
$p$ & \multicolumn{3}{c|}{Spectral Clustering \cite{ding2005equivalence}} & \multicolumn{3}{c|}{Spherical-Kmeans \cite{janssen2020k}} & \multicolumn{3}{c|}{MEXICO (POCS)} \\
\hline
& H & C & v-M & H & C & v-M & H & C & v-M \\
\hline
75 & 0.925$\pm$0.054 & 0.937$\pm$0.040 & 0.931$\pm$0.046 & 0.950$\pm$0.034 & 0.972$\pm$0.024 & 0.961$\pm$0.027 & \textbf{0.978}$\pm$0.025 & 0.976$\pm$0.024 & \textbf{0.977}$\pm$0.024 \\
\hline
100 & 0.918$\pm$0.058 & 0.934$\pm$0.039 & 0.926$\pm$0.048 & 0.943$\pm$0.031 & 0.967$\pm$0.024 & 0.955$\pm$0.026 & \textbf{0.976}$\pm$0.020 & \textbf{0.979}$\pm$0.021 & \textbf{0.976}$\pm$0.020 \\
\hline
150 & 0.889$\pm$0.060 & 0.925$\pm$0.031 & 0.906$\pm$0.045 & 0.940$\pm$ 0.026 & 0.962$\pm$0.020 & 0.951$\pm$0.022 & \textbf{0.973}$\pm$0.015 & \textbf{0.977}$\pm$0.013 & \textbf{0.975}$\pm$0.014\\
\hline
200 & 0.886$\pm$0.047 & 0.928$\pm$0.024 & 0.906$\pm$0.034 & 0.940$\pm$0.018 & 0.962$\pm$0.014 & 0.951$\pm$0.015 & \textbf{0.970}$\pm$0.015 & \textbf{0.975}$\pm$0.012 & \textbf{0.972}$\pm$0.013 \\
\hline
\end{tabular}}
\caption{Comparison of Homogeneity (H), Completeness (C) and v-Measure (v-M) from prediction scores for Sppectral Clusterings Spherical-Kmeans and Mexico with alternating projections on simulated data with different dimension $p$.}
\label{tab:tab4}
\end{table}

\begin{table}[h]
\centering
\resizebox{\textwidth}{!}{
\begin{tabular}{|c|ccc|ccc|ccc|}
\hline
$p$ & \multicolumn{3}{c|}{Spectral Clustering \cite{ding2005equivalence}} & \multicolumn{3}{c|}{Spherical-Kmeans \cite{janssen2020k}} & \multicolumn{3}{c|}{MEXICO (Dykstra)} \\
\hline
& H & C & v-M & H & C & v-M & H & C & v-M \\
\hline
75 & 0.925$\pm$0.054 & 0.937$\pm$0.040 & 0.931$\pm$0.046 & 0.950$\pm$0.034 & 0.972$\pm$0.024 & 0.961$\pm$0.027 & \textbf{0.977}$\pm$0.025 & 0.975$\pm$0.024 & \textbf{0.976}$\pm$0.024 \\
\hline
100 & 0.918$\pm$0.058 & 0.934$\pm$0.039 & 0.926$\pm$0.048 & 0.943$\pm$0.031 & 0.967$\pm$0.024 & 0.955$\pm$0.026 & \textbf{0.978}$\pm$0.020 & \textbf{0.979}$\pm$0.021 & \textbf{0.978}$\pm$0.020 \\
\hline
150 & 0.889$\pm$0.060 & 0.925$\pm$0.031 & 0.906$\pm$0.045 & 0.940$\pm$ 0.026 & 0.962$\pm$0.020 & 0.951$\pm$0.022 & \textbf{0.976}$\pm$0.015 & \textbf{0.980}$\pm$0.013 & \textbf{0.978}$\pm$0.014\\
\hline
200 & 0.886$\pm$0.047 & 0.928$\pm$0.024 & 0.906$\pm$0.034 &  0.940$\pm$0.018 & 0.962$\pm$0.014 & 0.951$\pm$0.015 & \textbf{0.967}$\pm$0.015 & \textbf{0.972}$\pm$0.012 & \textbf{0.970}$\pm$0.013 \\
\hline
\end{tabular}}
\caption{Comparison of Homogeneity (H), Completeness (C) and v-Measure (v-M) from prediction scores for Spherical-Kmeans and Mexico with Dykstra projection on simulated data with different dimension $p$.}
\label{tab:tab5}
\end{table}

\subsection{Additional Results on Anomaly detection}
\textbf{Real world data preprocessing.}\,
We present the details about the preprocessing of the real world datasets.

\begin{table}[h]
\centering
\begin{tabular}{|c|r|r|c|c|c|}
\hline
Dataset & \multicolumn{1}{c|}{Size} & \multicolumn{1}{c|}{Anomalies} & $p$  & $\tau$ & $\lambda$ \\ %& Extremes
\hline
SF & 73 237 & 3298 (4.5\%) & 4 & 0.8 & 10 \\
\hline
SA & 100 655 & 3377 (3.4\%) & 41 & 0.7 & 5 \\
\hline
http & 58 725 & 2209 (3.8\%) & 3 & 0.5 & 10 \\
\hline
shuttle & 49 097 & 3511 (7.2\%) & 9 & 0.7 & 5 \\
\hline
forestcover & 286 048 & 2747 (0.9\%) & 54 & 0.7 & 5 \\
\hline
\end{tabular}
\caption{Datasets Description and Parameter Configuration}
\label{tab_bonus}
\end{table}

The shuttle dataset is the fusion of the training and testing datasets available in the UCI repository \cite{Lichman2013}. The data have 9 numerical attributes, the first one being time. Labels from 7 different classes are also available. Class 1 instances are considered as normal, the others as anomalies. We use instances from all different classes but class 4, which yields an anomaly ratio (class 1) of 7.2\%.

In the forestcover data, also available at UCI repository \cite{Lichman2013}, the normal data are the instances from class 2 while instances from class 4 are anomalies, other classes are omitted, so that the anomaly ratio for this dataset is 0.9\%.

The last three datasets belong to the KDD Cup 99 dataset \cite{KDD99, Tavallaee2009}, produced by processing the tcpdump portions of the 1998 DARPA Intrusion Detection System (IDS) Evaluation dataset, created by MIT Lincoln Lab \cite{Lippmann2000}. The artificial data was generated using a closed network and a wide variety of hand-injected attacks (anomalies) to produce a large number of different types of attack with normal activity in the background. Since the original demonstrative purpose of the dataset concerns supervised AD, the anomaly rate is very high (80\%), which is unrealistic in practice, and inappropriate for evaluating the performance on realistic data. We thus take standard preprocessing steps in order to work with smaller anomaly rates.

For datasets SF and http we proceed as described in \cite{yamanishi2004line}: SF is obtained by picking up the data with positive logged-in attribute, and focusing on the intrusion attack, which gives an anomaly proportion of 4.5\%. The dataset http is a subset of SF corresponding to a third feature equal to ’http’. Finally, the SA dataset is obtained as in \cite{Eskin2002geometric} by selecting all the normal data, together with a small proportion (3.4\%) of anomalies.

\textbf{Further Experimental details on Anomaly Detection.}\,We present the full results of the performance of MEXICO regarding the anomaly detection task. The projection step is either performed using alternating projections based on the method POCS (Projection Onto Convex Sets) or with the more elaborate technique Dykstra.

\begin{table}[H]
\centering
\begin{tabular}{|c|c|c|c|c|}
\hline
Dataset & \shortstack{iForest \\ \cite{liu2008isolation}} & \shortstack{DAMEX \\ \cite{goix2016sparse}} & \shortstack{MEXICO \\ (POCS)} & \shortstack{MEXICO \\ (Dykstra)} \\
\hline
SF & 0.381$\pm$0.086 & 0.710$\pm$0.031 & \textbf{0.892}$\pm$0.013 &  0.710$\pm$0.030 \\
\hline
SA & 0.886$\pm$0.032 & 0.982$\pm$0.002 & 0.981$\pm$0.006 & 0.983$\pm$0.031 \\
\hline
http & 0.656$\pm$0.094 & 0.996$\pm$0.002 & 0.995$\pm$0.005 & 0.997$\pm$0.002 \\
\hline
shuttle & 0.970$\pm$0.020  & 0.990$\pm$0.003 & 0.990$\pm$0.003 & 0.989$\pm$0.003 \\
\hline
forestcover & 0.654$\pm$0.096 & 0.762$\pm$0.008 & \textbf{0.863}$\pm$0.015 &  0.851$\pm$0.008\\
\hline
\end{tabular}
\caption{Comparison of Area Under Curve of Receiver Operating Characteristic (ROC-AUC)from prediction scores of each method on different anomaly detection datasets.}

\label{tab6}
\end{table}

\begin{table}[H]
\centering
\begin{tabular}{|c|c|c|c|c|}
\hline
Dataset & \shortstack{iForest \\ \cite{liu2008isolation}} & \shortstack{DAMEX  \\ \cite{goix2016sparse}} & \shortstack{MEXICO \\ (POCS)} & \shortstack{MEXICO \\ (Dykstra)} \\
\hline
SF & 0.393$\pm$0.081 & 0.650$\pm$0.034 & \textbf{0.812}$\pm$0.016 & 0.661$\pm$0.031\\
\hline
SA & 0.879$\pm$0.031 &  0.938$\pm$0.012  & 0.940$\pm$0.031 & 0.950$\pm$0.011 \\
\hline
http & 0.658$\pm$0.099 & 0.968$\pm$0.009 & 0.972$\pm$0.012 & 0.971$\pm$0.008 \\
\hline
shuttle & 0.826$\pm$0.055 & 0.864$\pm$0.026 & 0.864$\pm$0.037 & 0.818$\pm$0.024 \\
\hline
forestcover & 0.894$\pm$0.037 & 0.893$\pm$0.010 &  0.958$\pm$0.006 & 0.954$\pm$0.004 \\
\hline
\end{tabular}
\caption{Comparison of Average Precision (AP) from prediction scores of each method on different anomaly detection datasets.}

\label{tab7}
\end{table}

\section{Further Numerical Experiments}\label{sec:Appendix_exp}
\subsection{MEXICO - Further Experimental Results}
\citet{cutler1994archetypal} provide an archetypal analysis of the Swiss Army dataset. This dataset consists of $6$ head dimensions from $200$ Swiss soldiers. The data was gathered to construct face masks for the Swiss army. Few samples of the dataset are presented in Table \ref{tab_swiss}.

The first measurement (MFB) corresponds to the width of the face just above the eyes. The second feature (BAM) corresponds to the width of the face just below the mouth. The third measurement (TFH) is the distance from the top of the nose to the chin. The fourth feature (LGAN) is the length of the nose. The fifth measurement (LTN) is the distance from the ear to the top of the head while the sixth  (LTG) is the distance from the ear to the bottom of the face. For a better visualization of the dataset, we made simple drawings of the different samples. Figure \ref{fig:face_features} illustrates the $6$ measurements.

\begin{table}[h!]
\centering
\begin{tabular}{|c|c|c|c|c|c|c|}
\hline
id & MFB & BAM & TFH & LGAN & LTN & LTG \\
\hline
0 & 113.2 & 111.7 & 119.6 & 53.9 & 127.4 & 143.6 \\
\hline
1 & 117.6 & 117.3 & 121.2 & 47.7 & 124.7 & 143.9 \\
\hline
2 & 112.3 & 124.7 & 131.6 & 56.7 & 123.4 & 149.3 \\
\hline
3 & 116.2 & 110.5 & 114.2 & 57.9 & 121.6 & 140.9 \\
\hline
4 & 112.9 & 111.3 & 114.3 & 51.5 & 119.9 & 133.5 \\
\hline
\end{tabular}
\caption{Extract of the Swiss Army dataset.}
\label{tab_swiss}
\end{table}
A question that naturally rises is to figure out subgroups of face features that get large simultaneously. MEXICO algorithm performed on the standardized dataset provides the following groups of features : $\{5,6\}$(green), $\{1, 3\}$(blue) and $\{2, 4\}$(red), as illustrated in Figure~\ref{fig:face_colored}. These results step in the direction of interpretability of feature clusters that may be large concomitantly in a real world dataset.
\begin{figure}[H]
     \centering
     \begin{subfigure}[b]{0.3\textwidth}
         \centering
         \includegraphics[width=\textwidth]{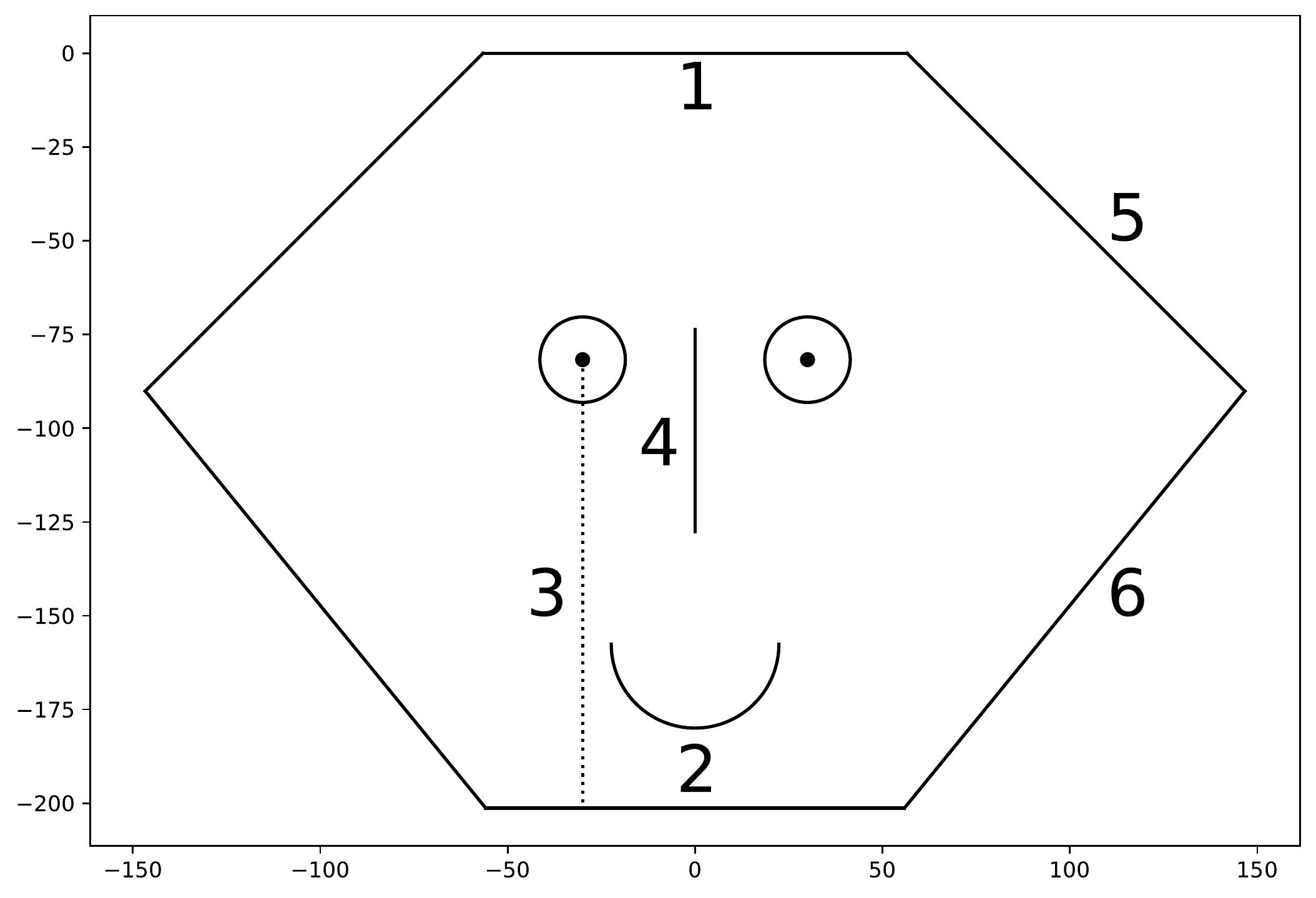}
         \caption{}
         \label{fig:face_features}
     \end{subfigure}
     \begin{subfigure}[b]{0.3\textwidth}
         \centering
         \includegraphics[width=\textwidth]{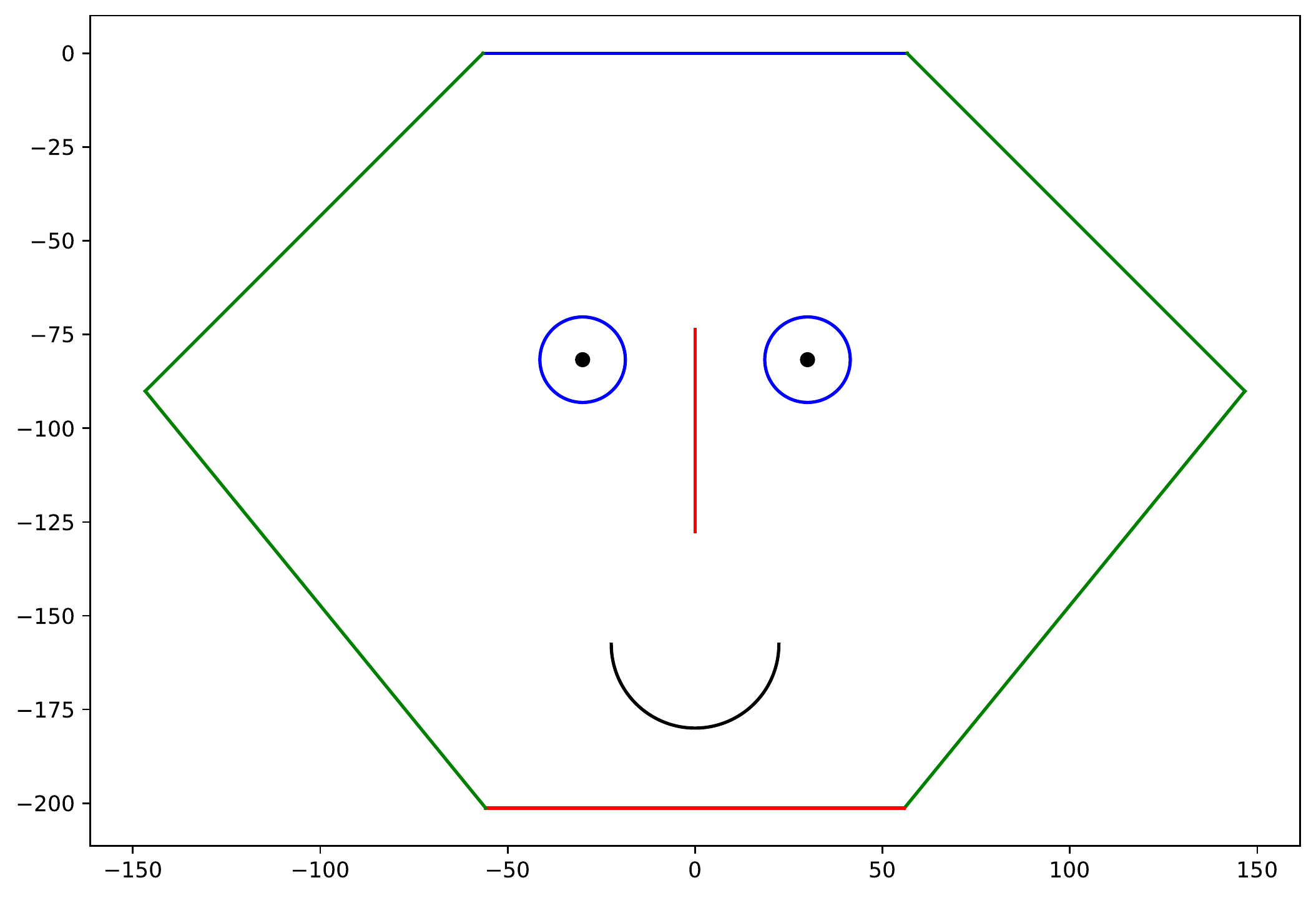}
         \caption{}
         \label{fig:face_colored}
     \end{subfigure}
     \caption{Illustration of the $6$ measurements (a) and subgroups that tend to be large simultaneously (b).}
\end{figure}

\subsection{Classification in Extremes after Dimension Reduction.}
\label{app:classification}
In this section, we compare the influence on a downstream classification task of dimension reduction resulting from MEXICO and two other methods applied to extreme samples, namely PCA and Sparse PCA. The considered task is binary classification in extremes. Following the experiments from \citet{jalalzai2018binary}, we generate $5\cdot10^3$ i.i.d samples in $\rset_+^4$ following a logistic distribution with dependence parameter $\delta=0.2 $. These points are labeled $+1$. Similarly we  generate $5\cdot 10^3$ i.i.d samples in $\rset_+^4$ following a logistic distribution with parameter $\delta =0.4 $. These points are labeled $-1$. The samples are projected in lower-dimensional subspace according to methods summarized in Table~\ref{tab:dim_reduction}. Random Forest is the considered class of classifiers. The number of trees is set to $10$. The norm is the $\ell_\infty$ norm. The value of $k$ is set to $100$. In Figure \ref{fig:dimRed_RF}, boxplots obtained over $100$ runs show the  test error rate  in extremes for a Random Forest classifier after any dimension reduction method from Table ~\ref{tab:dim_reduction}. %Mann-Whitney rank test performed on the underlying classification losses associated to each dimension reduction scheme prove that the MEXICO's influence on classification is not statistically significant when compared to PCA and Sparse PCA.
\begin{table}[h!]
\centering
\begin{tabular}{|c|c|c|c|c|}
\hline
Method & Initial dimension &  Resulting dimension & Sparsity & Extreme Dependence Structure\\
\hline
%Standardization & $p$ & $p-1$ & &  \\
\hline
PCA & $p$ & $m$ &  &\\
\hline
Sparse PCA & $p$ & $m$ & \checkmark & \\
\hline
MEXICO & $p$ & $m$ & \checkmark & \checkmark \\
\hline
\end{tabular}
\caption{Summary of Dimension Reduction Methods in Extremes.}
\label{tab:dim_reduction}
\end{table}

\begin{figure}[H]
    \centering
    \includegraphics[width=0.4\textwidth]{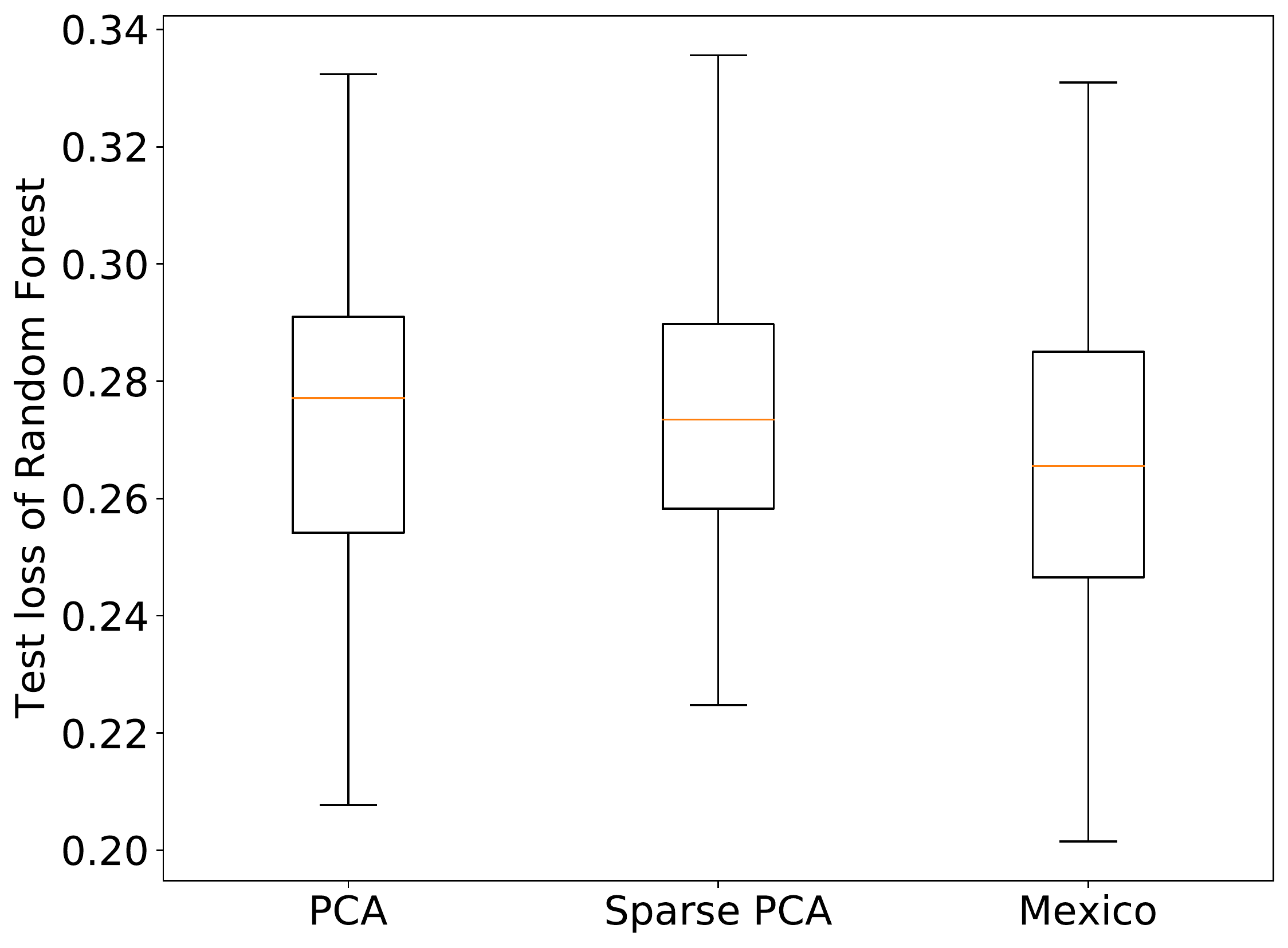}
    \caption{Classification Error in Extremes after Dimension Reduction.}
    \label{fig:dimRed_RF}
\end{figure}

\subsection{Contraction mapping}
\label{app:contractionmapping}
Figure \ref{fig:contraction} depicts the contraction mapping property induced by MEXICO in the following example: Logistic data (see Section \ref{logistic_Appendix} above) is generated in $\rset^4$ with two different dependence structures $K_1 = \{1,2\}$ and $K_2= \{3,4\}$ thus the total number of clusters $m$ is equal to $2$. The number of generated samples is $n=10^6$ and $k$ is set to $\sqrt{n}$. The violin plot on the left represents the distributional distance between $2$ standardized random extreme samples after normalization $\theta_i = X_i/\|X_i\|_\infty, \theta_2 = X_2/\|X_2\|_\infty$
while the violin plot on the right reports the distances between $2$ transformed and normalized data $\widetilde\theta_i = \widetilde X_i/\|\widetilde X_i\|_\infty, \widetilde \theta_2 = \widetilde X_2/\|\widetilde X_2\|_\infty$. All distances are computed on normalized data for the sake of visualization. The lower (\textit{resp. upper}) part of the violin plot depicts the intra-cluster (\textit{resp. intra-cluster})
 $\ell_2$ distance.\\
 First, We focus on the lower part of the violin plots: the distribution of intra-cluster distance is twice as big with the original data when compared to the transformed data. Second, it is worth noting that the inter cluster distances (\textit{i.e.} upper part of the violin plot) tends to be smaller after transformation when compared with the original data.
\begin{figure}[H]
    \centering
    \includegraphics[width=0.35\textwidth]{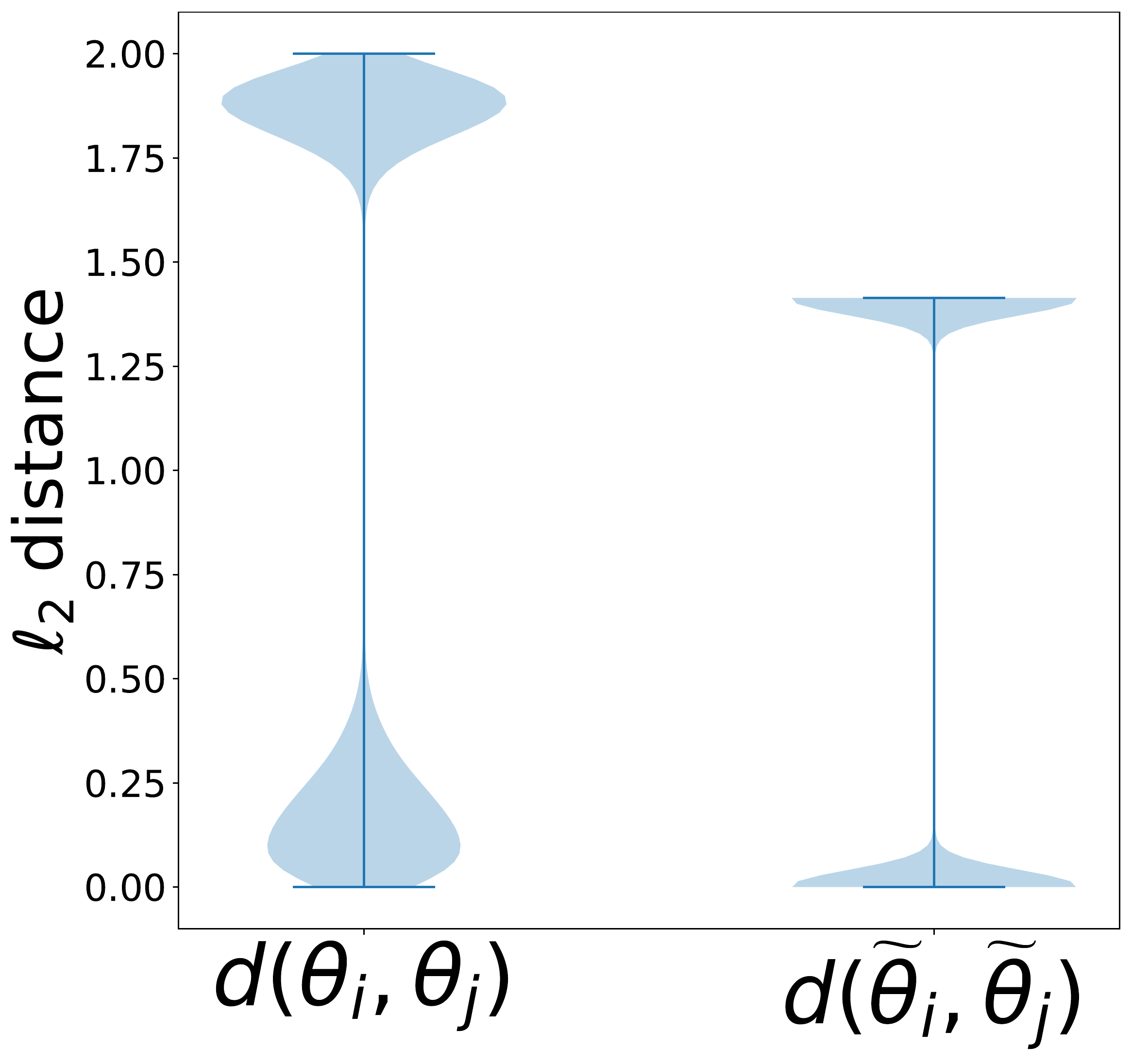}
    \caption{(Left)Violin plots of $l_2$ pairwise distances between normalized extreme samples $\theta_i, \theta_j$ and (Right) corresponding pairwise distances for transformed samples $\widetilde\theta_i, \widetilde\theta_j$.}
    \label{fig:contraction}
\end{figure}

\subsection{Angular MEXICO}
\label{app:angular_Mexico}
Theorem \ref{th3} provides a bound of the excess risk of MEXICO algorithm when working with normalized data. As suggested by the correspondance between the angular and exponent measures (see Definitions \ref{sec:EVT}), the analysis of the dependence structure of extremes clusters of features with the \emph{normalized} data (thus relying on the angular measure $\Phi$) is equivalent to the analysis with the \emph{non-normalized} data (thus relying on the exponent measure $\mu$).  Algorithm \ref{alg:Mexico_angular} is the counterpart of MEXICO as detailed in Algorithm \ref{alg:Mexico} relying on the normalized extremes (\textit{i.e.} $\{\Theta(X_{(i)})\}_{i\leq k}$). It follows \citet{janssen2020k} and exploits the normalized data to assess the dependence structure of extremes.

Input data or normalized data provide similar output matrix $\mathbf{W}_{mex}$ as illustrated in Figure \ref{fig:matrix_mex} when dealing with the dataset detailed in Section \ref{app:contractionmapping}. The $\mathbf{W}_{mex}$ matrix associated to Algorithm \ref{alg:Mexico} or Algorithm \ref{alg:Mexico_angular} both recovers the feature dependence structure of extremes as the first column of both matrices is associated to the cluster $\{0,1\}$ while the second column of both matrices is associated to the cluster $\{2,3\}$. Note that working on angular data implies normalizing samples which adds complexity to solve our problem of interest.

\begin{algorithm}[!h]
%  \label{algo:Mexico}
\caption{Angular MEXICO}
%\algsetup{linenodelimiter=.}
\begin{algorithmic}[1]
\REQUIRE Training data $(X_1,\ldots,X_n), 0<m<p, \lambda>0$ and rank $k(=\lfloor n \gamma\rfloor)$.
\STATE Initialize $(\mathbf{W}_0,\mathbf{Z}_0) \in \mathcal{A}_p^m \times \mathcal{A}_m^n$.
%\STATE Compute the index of extreme samples $\mathcal{I} =\{ i \in  \llbracket 1,n \rrbracket, ||\widehat{V}_i|| \geq n/k \}$.
\STATE Standardize the data $\widehat V_{(i)} = \widehat T(X_{(i)})$ (see Remark \ref{rmk:ParetoStandardization}).% \geq \ldots \geq  \|X_{(n)}\|$
\STATE Sort training data by decreasing order of magnitude $ \|\widehat V_{(1)}\|_\infty \geq \ldots \geq  \|\widehat V_{(n)}\|_\infty$.
%\STATE Compute the index of extreme samples $\mathcal{I} =\{ i \in  \llbracket 1,n \rrbracket, ||\widehat{V}_i|| \geq n/k \}$.
\STATE Consider the set of $k$ extreme training data  $\widehat V_{(1)}, \ldots, \widehat V_{(k)}$.
\STATE Normalize the $k$ extreme data $ \theta_i = \widehat V_{(i)}/ \|\widehat V_{(i)}\|_\infty$.
%\STATE Normalize extremes $\Theta = \big\{\widehat V_{(i)}\big / \|\widehat V_{(i)}\|_\infty \}_{(i \leq k)}$. %\in \rset^{n \times p}_+$ .
\STATE Compute $(\mathbf{W}_{mex},\mathbf{Z}_{mex}) \in \argmax_{(\mathbf{W},\mathbf{Z})} f_{\lambda}(\mathbf{W},\mathbf{Z})$ using update rule  \eqref{eq:algorithm} on the normalized data.
\STATE Given a new input ${X}_{\text{new}}$ standardized as $\widehat V_{\text{new}}$ with $\|\widehat V_{\text{new}}\|_\infty \geq \|\widehat V_{(k)}\|_\infty$ and normalized as $\theta_{\text{new}}$\\ compute $\widetilde \theta_{\text{new}}=\theta_{\text{new}} \mathbf{W}_{mex}$.
\STATE Compute predicted cluster $\varphi_0 = \argmax_{1 \leq j \leq m}{\widetilde \theta _{\text{new}}^j}$.
\STATE (\textbf{FC}) Return cluster $\varphi_0$. \\ (\textbf{AD}) Return score $\ell(\widetilde \theta_{\text{new}},\mathbf{W}_{{mex}})$.
\end{algorithmic}
\label{alg:Mexico_angular}
\end{algorithm}

\begin{figure}[h]{}
  \centering
  \begin{subfigure}[t]{0.3\textwidth}
    \includegraphics[width=\textwidth]{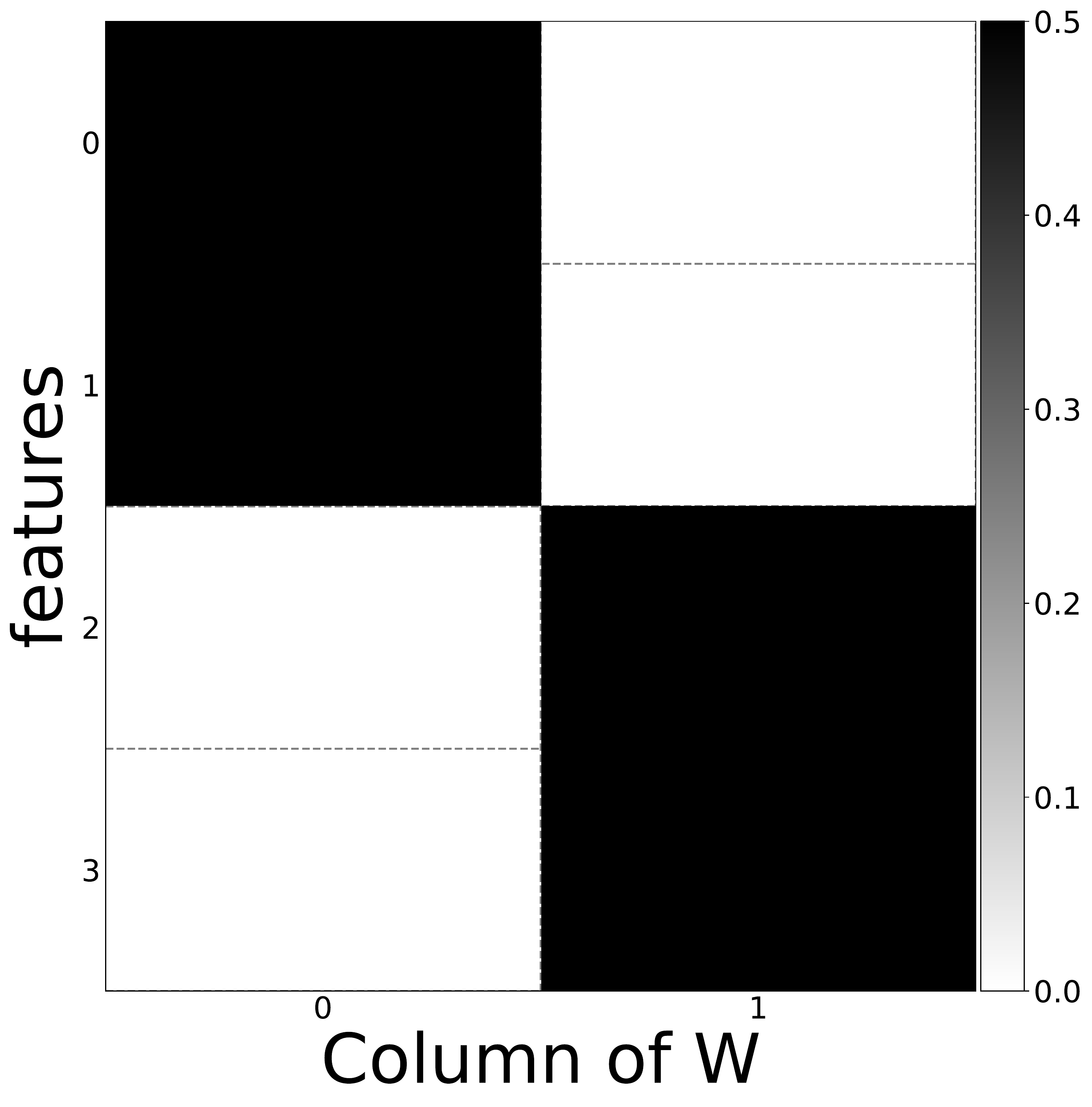}
    \caption{MEXICO - Algorithm~\ref{alg:Mexico}}
    \label{fig:mexmex}
    \end{subfigure}
    \qquad
\begin{subfigure}[t]{0.3\textwidth}
    \includegraphics[width=\textwidth]{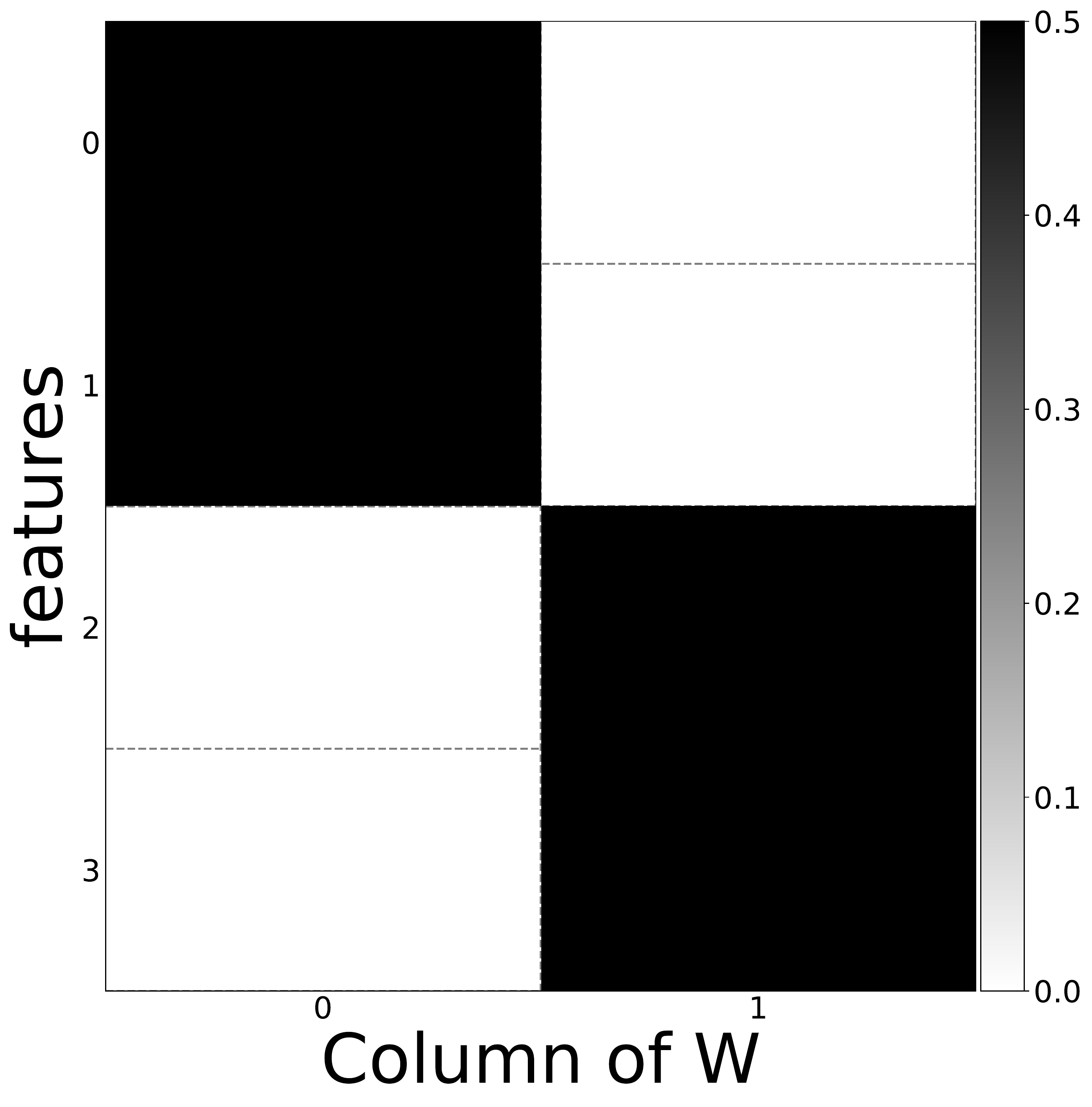}
    \caption{Angular MEXICO - Algorithm~\ref{alg:Mexico_angular}}
    \label{fig:mexangular}
\end{subfigure}
\caption{Matrix $\mathbf{W}_{mex}$ of Algorithm \ref{alg:Mexico} (\ref{fig:mexmex}) and  Algorithm \ref{alg:Mexico_angular} (\ref{fig:mexangular}).}
\label{fig:matrix_mex}
\end{figure}

\textbf{Conclusions on experimental findings.}  Hereafter, we summarize the empirical findings detailed above:\\
First, working with extremes or their normalized counterparts provide similar results to estimate the support of extremes with MEXICO algorithm as illustrated in Section \ref{app:angular_Mexico}. This empirical evidence is expected because of the \emph{one-to-one} correspondance between the angular measure and the exponent measure.
Second, while we make no claim that MEXICO outperforms PCA or Sparse PCA, it is worthy of attention that the dimension reduction induced by MEXICO remains competitive on a downstream classification task as detailed in Section \ref{app:classification}.
%MEXICO algorithm forms a matrix $\mathbf{W}_{mex}$.
Finally, the transformation induced by MEXICO can be considered as a contraction mapping as illustrated in Section \ref{app:contractionmapping} for disjoint clusters. Future work will explore the analysis of MEXICO algorithm on overlapping clusters.

%\textbf{Conclusion.}\ We summarize the empirical findings from this section:
%\begin{itemize}
%  \item
%  \item
%\end{itemize}

\end{document}